\documentclass[11pt]{article}

\usepackage[margin=1in]{geometry}
\usepackage[T1]{fontenc}
\usepackage{times}
\usepackage{amsmath,amssymb,amsthm,mathtools}
\usepackage{booktabs,longtable,array,multirow}
\usepackage{enumitem}
\usepackage{xcolor}
\usepackage{graphicx}
\usepackage{float}
\usepackage[ruled,vlined,algo2e]{algorithm2e}
\usepackage{microtype}
\usepackage{natbib}
\usepackage{url}
\usepackage{hyperref}
\usepackage{prettyref}

\hypersetup{
  colorlinks=true,
  linkcolor=blue!55!black,
  citecolor=blue!55!black,
  urlcolor=blue!55!black
}

\allowdisplaybreaks
\setlist{leftmargin=2em,itemsep=0.15em,topsep=0.25em}

\newtheorem{theorem}{Theorem}[section]
\newtheorem{lemma}[theorem]{Lemma}

\newtheorem{assumption}[theorem]{Assumption}

\newcommand{\E}{\mathbb{E}}
\newcommand{\Prob}{\mathbb{P}}
\newcommand{\one}{\mathbf{1}}
\newcommand{\Reg}{\operatorname{Reg}}
\newcommand{\order}{\mathcal{O}}
\newcommand{\otil}{\widetilde{\order}}
\newcommand{\argmax}{\operatorname*{arg\,max}}

\newcommand{\dd}{\mathrm{d}}
\newcommand{\calP}{\mathcal{P}}
\newcommand{\calF}{\mathcal{F}}
\newcommand{\calG}{\mathcal{G}}

\newcommand{\pref}[1]{\prettyref{#1}}

\newcommand{\savehyperref}[2]{\texorpdfstring{\hyperref[#1]{#2}}{#2}}
\newrefformat{eq}{\savehyperref{#1}{Eq.~\textup{(\ref*{#1})}}}
\newrefformat{eqn}{\savehyperref{#1}{Eq.~(\ref*{#1})}}
\newrefformat{lem}{\savehyperref{#1}{Lemma~\ref*{#1}}}
\newrefformat{def}{\savehyperref{#1}{Definition~\ref*{#1}}}
\newrefformat{line}{\savehyperref{#1}{Line~\ref*{#1}}}
\newrefformat{thm}{\savehyperref{#1}{Theorem~\ref*{#1}}}
\newrefformat{corr}{\savehyperref{#1}{Corollary~\ref*{#1}}}
\newrefformat{cor}{\savehyperref{#1}{Corollary~\ref*{#1}}}
\newrefformat{sec}{\savehyperref{#1}{Section~\ref*{#1}}}
\newrefformat{app}{\savehyperref{#1}{Appendix~\ref*{#1}}}
\newrefformat{assum}{\savehyperref{#1}{Assumption~\ref*{#1}}}
\newrefformat{asm}{\savehyperref{#1}{Assumption~\ref*{#1}}}
\newrefformat{ex}{\savehyperref{#1}{Example~\ref*{#1}}}
\newrefformat{fig}{\savehyperref{#1}{Figure~\ref*{#1}}}
\newrefformat{alg}{\savehyperref{#1}{Algorithm~\ref*{#1}}}
\newrefformat{rem}{\savehyperref{#1}{Remark~\ref*{#1}}}
\newrefformat{conj}{\savehyperref{#1}{Conjecture~\ref*{#1}}}
\newrefformat{prop}{\savehyperref{#1}{Proposition~\ref*{#1}}}
\newrefformat{proto}{\savehyperref{#1}{Protocol~\ref*{#1}}}
\newrefformat{prob}{\savehyperref{#1}{Problem~\ref*{#1}}}
\newrefformat{claim}{\savehyperref{#1}{Claim~\ref*{#1}}}
\newrefformat{que}{\savehyperref{#1}{Question~\ref*{#1}}}
\newrefformat{op}{\savehyperref{#1}{Open Problem~\ref*{#1}}}
\newrefformat{fn}{\savehyperref{#1}{Footnote~\ref*{#1}}}
\newrefformat{eve}{\savehyperref{#1}{Event~\ref*{#1}}}
\newrefformat{tab}{\savehyperref{#1}{Table~\ref*{#1}}}

\RestyleAlgo{ruled}
\SetAlgoVlined
\SetKwInput{KwInput}{Input}
\SetKwInput{KwInitialize}{Initialize}

\title{Optimal Nonparametric Dynamic Pricing with Censored Demand and Adversarial Inventory}
\author{
Mengxiao Zhang\thanks{University of Iowa; \href{mailto:mengxiao-zhang@uiowa.edu}{mengxiao-zhang@uiowa.edu}}
\and
Yingfei Wang\thanks{University of Washington; \href{mailto:yingfei@uw.edu}{yingfei@uw.edu}}
\and
Haipeng Luo\thanks{University of Southern California; \href{mailto:haipengl@usc.edu}{haipengl@usc.edu}}
}
\date{}

\begin{document}
\maketitle

\begin{abstract}
We study online dynamic pricing with censored demand, where an arbitrary inventory level is revealed before pricing and may adapt to past observations, while demand follows an unknown, price-dependent distribution that is stationary over time.
For a horizon of $T$ rounds, \citet{xu2026dynamic} achieved $\otil(\sqrt{T})$ regret under restrictive structural assumptions including linear demand, price-independent additive noise, and conditions relating inventory levels to the noise support. Our first contribution is to extend this framework to a substantially more general and statistically harder nonparametric setting, requiring only the natural assumption that expected sales are nonincreasing in price and allowing nonlinear demand curves and price-dependent noise. For this model, we first propose a simple baseline, Double-Grid-UCB, which discretizes both price and inventory and achieves $\otil(T^{3/4})$ expected regret using separate revenue estimates for each price-inventory grid pair.
Then, we develop Threshold-UCB, which improves the expected regret to $\otil(T^{2/3})$. Unlike Double-Grid-UCB, Threshold-UCB reuses sales observations across inventory levels through shared estimates of demand-tail probabilities, allowing the same data to support revenue upper bounds for multiple inventories rather than a single inventory bin. We also complement this upper bound with an $\Omega(T^{2/3})$ lower bound via a reduction from stochastic posted pricing, establishing its minimax optimality. Finally, extensive experiments across inventory processes, demand functions, and noise models demonstrate consistently superior performance of Threshold-UCB over benchmark algorithms.

\end{abstract}

\section{Introduction}

Dynamic pricing is a central tool in revenue management because price controls
both the revenue earned per unit and the quantity demanded. Classical models characterize optimal prices when the demand law is known \citep{gallego1994optimal}. Subsequent work has studied how a seller can learn an unknown demand curve while simultaneously earning revenue \citep{kleinberg2003value,besbes2009dynamic,broder2012dynamic,keskin2014dynamic}, including contextual formulations in which demand depends on observed customer or product characteristics \citep{javanmard2019dynamic,xu2021logarithmic,luo2022contextual,xu2022towards}. 
In these models, inventory is typically unlimited, fixed in advance, or controlled by the seller.

In many applications, however, available inventory is limited, varies across selling periods, and is outside the seller's control. Consider a retailer receiving a highly perishable crop from a nearby farm. Daily supply depends on when the crop ripens and can vary sharply, while unsold units spoil before the next period. 
From the retailer’s perspective, this motivates treating the inventory sequence as arbitrary: in each period, the retailer observes the available inventory and then chooses a price. Demand, however, is censored by inventory: when a stockout occurs, the retailer observes only the number of units sold, 
but not the underlying demand.
Beyond censoring demand, the varying inventory also affects optimal pricing: the revenue-maximizing price depends on the available inventory, so the relevant benchmark selects a different price for each observed inventory level.

Motivated by such applications, \citet{xu2026dynamic} consider an adversarial-inventory model and propose an algorithm called C20CB, which achieves $\otil(\sqrt T)$ regret after $T$ rounds and is minimax optimal for their structured model up to logarithmic factors.\footnote{We use $\otil(\cdot)$ to suppress logarithmic factors.} Their result, however, relies heavily on several restrictive assumptions. First, potential demand must be linear in price with fixed coefficients, plus additive noise. Second, the noise must be bounded, zero-mean, and i.i.d., with a common price-independent law and a Lipschitz CDF. Third, inventory must satisfy support and interiority conditions that ensure complete censoring at the lowest price and uncensored demand at the highest price. The latter guarantees that true demand, rather than only sales truncated by inventory, is observed at the highest price. Together, these assumptions enable estimation of the linear demand parameters and common noise distribution. However, the resulting algorithm and analysis do not directly extend to nonlinear demand or price-dependent noise, which arise naturally in practice. This leaves the following question open.

\begin{center}
\vspace{0.2em}
\emph{Can one obtain sublinear regret for adversarial-inventory dynamic pricing
under nonparametric monotone demand in general? If so, what is the minimax-optimal
regret rate?}
\vspace{0.2em}
\end{center}

\paragraph{Contributions.}
We provide a complete answer to this question.
Our first contribution is to introduce
a substantially more general, nonparametric adversarial-inventory model, which requires only the
natural assumption that expected sales are nonincreasing in price, without imposing an additive-noise
representation or other parametric structure.
We then propose two algorithms for this model: the first, simpler algorithm achieves a suboptimal regret of $\otil(T^{3/4})$, while the second improves the rate to $\otil(T^{2/3})$. Both algorithms follow the upper confidence bound (UCB) principle, choosing the price with the largest optimistic estimate of its revenue under the current inventory.
Specifically,

\begin{itemize}[leftmargin=*]
\item In \pref{sec:double-grid}, we introduce Double-Grid-UCB, which discretizes both price and inventory and uses separate revenue estimates for each price-inventory pair to select prices optimistically. It achieves $\otil(PDT^{3/4})$ expected regret (\pref{thm:double-grid-upper}), where $P$ and $D$ are the largest possible price and inventory, respectively. However, learning each inventory bin separately leaves useful information from other inventory levels unused.

\item To improve this rate, in \pref{sec:threshold-construction}, we exploit the common demand distribution across inventory levels at each price. Specifically, we first consider the classical Kaplan-Meier estimator and show that its  confidence guarantee \emph{does not} directly apply to our setting with adaptive censoring. We therefore develop Threshold-UCB, which shares each sales observation across multiple demand thresholds and achieves $\otil(PDT^{2/3})$ expected regret (\pref{thm:regret-upper}). A matching $\Omega(T^{2/3})$ lower bound for $P=D=1$ establishes minimax optimality up to logarithmic factors (\pref{thm:lower}).

\item Finally, in \pref{sec:experiments}, we compare Threshold-UCB with C20CB~\citep{xu2026dynamic} and other benchmarks across linear and nonlinear demand, varying inventories, and additive and multiplicative noise. Threshold-UCB is competitive with C20CB in the regime covered by C20CB's theory and achieves lower mean regret in most tested regimes outside that regime.
\end{itemize}
\subsection{Related Work}\label{sec:related}

\paragraph{Dynamic pricing and posted-price learning.}
Dynamic pricing is a classical problem in revenue management, with foundational work on finite inventories under a known stochastic demand law \citep{gallego1994optimal}. When demand is unknown, \citet{kleinberg2003value} study online posted pricing and establish regret upper and lower bounds under identical, i.i.d., or adversarial buyer valuations. Our lower bound builds directly on their $\Omega(T^{2/3})$ construction under i.i.d. valuations, which we embed into our model through constant inventory and binary demand. Research on pricing and learning also considers different demand structures, ranging from linear models \citep{keskin2014dynamic} to general parametric choice models \citep{broder2012dynamic} and nonparametric demand \citep{besbes2009dynamic}. Contextual pricing further incorporates customer and product features into pricing decisions, including work on linear valuation models under known noise distributions \citep{javanmard2019dynamic,xu2021logarithmic} and extensions to unknown noise distributions under weaker assumptions \citep{luo2022contextual,xu2022towards}. 

\paragraph{Learning from lost sales.}
Censoring also plays a central role in data-driven inventory control, where sales reveal demand only up to the available inventory. Its implications for learning have been studied in the newsvendor problem
\citep{besbes2013implications}. Bayesian approaches update demand beliefs from censored sales \citep{chen2010bounds,bisi2011censored}, while nonparametric approaches use stochastic
gradient methods or Kaplan--Meier estimation to construct adaptive base-stock policies without observing unmet demand \citep{huh2009adaptive,huh2011adaptive}. When prices are also decision variables, joint pricing and inventory control
methods learn nonparametric demand while choosing both prices and stocking levels
\citep{chen2021nonparametric,chen2024optimal}. Within this joint-decision setting,
\citet{chen2020databased} additionally consider a constraint on the number of price changes under a parametric demand model. In contrast to this literature, our learner does not choose inventory
or use inventory decisions to control censoring.
Instead, inventory is revealed as a context before pricing and may adapt to past sales, while the learner chooses a price based on the current inventory.

\paragraph{Adversarially censored dynamic pricing.}
Closest to our work is \citet{xu2026dynamic}, which studies the same within-period
order of inventory, price, and censored sales and permits an adversarial
inventory sequence.  Their algorithm achieves $\otil(\sqrt T)$ regret under a more restrictive model that assumes linear mean demand, a common bounded
price-independent additive-noise CDF, regularity of that CDF, and inventory
support conditions.  The shared additive shock makes potential demand
automatically nonincreasing in price and enables observations at different
prices to be pooled to estimate common demand parameters and a common noise law.  
Our work removes this parametric structure by considering a broader nonparametric model, for which the worst-case regret rate is necessarily $T^{2/3}$. Thus, their $\Omega(\sqrt{T})$ lower bound and our $\Omega(T^{2/3})$ lower bound apply to different model classes and do not conflict.
See \pref{sec:comparison} for
detailed comparisons between the two models.

\section{Problem Formulation}\label{sec:model}

\paragraph{Notation.}
For a positive integer $n$, let $[n]\triangleq\{1,\ldots,n\}$.  We denote
$\one\{A\}$ for the indicator of an event $A$, and use $\E[\cdot]$ and
$\Prob(\cdot)$ for expectation and probability.  For $x\in\mathbb R$, let
$x_+\triangleq\max\{x,0\}$.

\paragraph{Problem setting.}
The learner and the environment interact for $T\geq 2$ rounds.  Let
$\calF_{t-1}$ denote the sigma-algebra generated by the learner's internal
random seed and all inventories, prices, and sales observed through round
$t-1$.  At the beginning of round $t$, the environment reveals an inventory
level $\gamma_t\in[0,D]$, where $D>0$ is the known maximum inventory.  The
conditional law of $\gamma_t$ may depend on $\calF_{t-1}$, allowing inventory
to adapt to the observed past.  The learner then selects a price
$p_t\in[0,P]$ using $\calF_{t-1}$ and $\gamma_t$, where $P>0$ is the maximum
allowed price.

After the price is posted, a nonnegative potential demand $Y_t(p_t)$ is
realized, where $Y_t(p)$ denotes demand at price $p$ before inventory
censoring.  The learner observes the sale
$X_t\triangleq\min\{Y_t(p_t),\gamma_t\}$ and earns realized revenue $r_t\triangleq p_tX_t$.
For the analysis, define the pre-demand sigma-algebra
$\calG_t\triangleq\sigma(\calF_{t-1},\gamma_t,p_t)$.  The random functions
$Y_t(\cdot):[0,P]\mapsto\mathbb{R}^+$ have the same law in every round, and
$Y_t(\cdot)$ is independent of $\calG_t$.  Thus, inventory and price may
depend on the observed past, but neither can anticipate the current demand
shock.  We impose no restriction on dependence across prices within a round,
as sales feedback is available only at the posted price.

For $p\in[0,P]$, define the demand CDF and survival function by
$F_p(u)\triangleq\Prob(Y_t(p)\le u)$ and
$S_p(u)\triangleq1-F_p(u)=\Prob(Y_t(p)>u)$, respectively.  At inventory
$\gamma$, we define the expected sales and revenue by
$
s(p,\gamma)
\triangleq\E\left[\min\{Y_t(p),\gamma\}\right]
=\int_0^\gamma S_p(u)\,\dd u$ and $
R(p,\gamma)\triangleq p\cdot s(p,\gamma).
$
The integral representation follows from the layer-cake identity for a
nonnegative random variable.

The only structural assumption we make is the following mild and natural monotonicity condition on expected sales.
\begin{assumption}[Monotone expected sales]\label{assum:monotone-demand}
For every inventory level $\gamma\in[0,D]$, the expected sales $s(p,\gamma)$ is nonincreasing in
$p$ on $[0,P]$.  Equivalently, $s(p,\gamma)\ge s(p',\gamma)$ for every
$0\le p\le p'\le P$.
\end{assumption}

\pref{assum:monotone-demand} is implied by first-order stochastic dominance of potential demand across prices. Specifically, if $p\le p'$ implies $S_p(u)\ge S_{p'}(u)$ for every $u\ge0$, then integrating this inequality over $[0,\gamma]$ gives the required monotonicity. In particular, the assumption holds whenever potential demands at different prices can be constructed using the same underlying randomness so that $p\mapsto Y_t(p)$ is nonincreasing almost surely. For example, if
$
Y_t(p)=\mu(p)+\epsilon_t
$
where $\mu(p)$ is nonincreasing in $p$, then using the same realization of $\epsilon_t$ for all prices ensures that $Y_t(p)\ge Y_t(p')$ whenever $p\le p'$. \pref{assum:monotone-demand} is used only for the later continuous-to-discrete approximation, and our model requires no smoothness, parametric form, or additional distributional relation for demand.

\paragraph{Goal.}
Our goal is to minimize expected regret relative to an oracle that knows the demand
law and optimizes expected revenue separately for each realized inventory.
More concretely, we define regret as the oracle's cumulative expected revenue minus that of
the learner:
\begin{align}
\Reg_T
\triangleq\sum_{t=1}^T
\sup_{p\in[0,P]}R(p,\gamma_t)-\sum_{t=1}^TR(p_t,\gamma_t).
\label{eq:continuous-regret}
\end{align}
All our algorithms use the uniform price grid defined by
$\calP_K\triangleq\{p_m=(m-1)P/K:m\in[K+1]\}$, which contains $K+1$ prices.
The integer $K\ge1$ may differ across algorithms. Since $\gamma_t\le D$ for all $t$, we have
$\sum_{t=1}^T\gamma_t\le DT$.
By \pref{assum:monotone-demand}, restricting the oracle to $\calP_K$ incurs
an additive approximation error of at most $PDT/K$.
Specifically, for $p\in[0,P]$, let
$\underline p$ be the largest grid price no greater than $p$.  By
\pref{assum:monotone-demand}, $s(\underline p,\gamma)\ge s(p,\gamma)$.
Moreover, $p-\underline p\le P/K$ and $s(p,\gamma)\le\gamma$, so $
R(\underline p,\gamma)
\ge \underline p\,s(p,\gamma)
\ge R(p,\gamma)-\frac{P\gamma}{K}.$ Taking the maximum over all prices in each round gives
\begin{align}
\Reg_T
\le \frac{PDT}{K}
+\sum_{t=1}^T
\left[\max_{m\in[K+1]}R(p_m,\gamma_t)-R(p_t,\gamma_t)\right]
\triangleq \frac{PDT}{K}+\Reg_T^{\calP_K}.
\label{eq:regret-decomposition}
\end{align}
The first term is the price-grid approximation error, while
$\Reg_T^{\calP_K}$ is the regret relative to the best grid price at each round. We emphasize again that this continuous-to-discrete approximation is the only place where \pref{assum:monotone-demand} is utilized, and the rest of the analysis does not require it.

\subsection{Relation to the linear-additive model of~\citet{xu2026dynamic}.}\label{sec:comparison}
In this section, we discuss how our model is substantially more general and challenging than that of~\citet{xu2026dynamic}.
In our notation, their model assumes
$Y_t(p)=a-bp+N_t$ for $p\in[0,P]$, with unknown parameters
$0<a\le a_{\max}$ and $0<b_{\min}\le b\le b_{\max}$.
The additive noise variables $N_t$ are i.i.d. with zero mean and follow a common,
price-independent distribution supported on $[-c,c]$ with a Lipschitz CDF.
The constants $a_{\max},b_{\min},b_{\max},c,P$ and additionally an inventory lower
bound $\gamma_{\min}$ are assumed to be known.  Their inventory conditions require, for
every round $t$,
\begin{align}
2c<\gamma_t<a-c,
\qquad
\gamma_t\ge\gamma_{\min}>
a_{\max}-b_{\min}P+c.
\label{eq:linear-inventory-support}
\end{align}
They additionally require $a-bP-c>0$ and
$P\ge a/(2b)$.
Thus, inventory exceeds the width of the noise support but remains below the
smallest possible demand at price zero.  At the maximum price $P$, by
contrast, even the largest possible demand lies strictly below every
admissible inventory.
Note that the condition $b>0$
also ensures that \pref{assum:monotone-demand} holds.
Their C20CB algorithm achieves $\otil(\sqrt T)$ regret under these assumptions when the horizon is sufficiently large.  

These conditions are restrictive, and they also make learning much easier. The known gap in \pref{eq:linear-inventory-support} guarantees a nonempty price interval ending at $P$ on which demand is never censored, whatever the inventory. Uncensored observations at two prices in this interval identify $a$, $b$, and the noise distribution. A linear shift of these observations then gives an unbiased revenue estimate at every price and inventory level, so censoring never has to be handled (see \pref{app:xu-comparison} for details). Our model, by contrast, assumes only monotone expected sales and imposes no additive-noise
representation, allowing nonlinear, heteroskedastic, and price-dependent demand distributions. The $\Omega(T^{2/3})$ lower bound in \pref{app:lower} shows the statistical cost of this broader class, even when inventory is constant and censoring disappears.

\section{Algorithm}
\label{sec:algorithm}
In this section, we introduce two upper confidence bound (UCB) algorithms for our problem. As a warm-up, in
\pref{sec:double-grid}, we introduce Double-Grid-UCB, a simple baseline that discretizes
both price and inventory, maintains separate revenue estimates for each price-inventory grid pair, and achieves $\otil(PDT^{3/4})$ expected regret.
In \pref{sec:threshold-construction}, we develop Threshold-UCB, which improves this rate to
$\otil(PDT^{2/3})$ by reusing sales observations across inventory levels through shared estimates of demand-tail probabilities. Below, we first present a general UCB framework and its regret analysis that are shared by the two algorithms.

Fix the price grid
$\calP_K=\{p_m=(m-1)P/K:m\in[K+1]\}$.
After observing $\gamma_t$ at round $t$, the algorithm computes an index
$\mathsf U_t(m,\gamma_t)$ for each price $p_m\in\calP_K$. Suppose that we can
construct a uniform confidence event $\mathcal E$ with
$\Pr(\mathcal E^c)\le 1/T$, on which
\begin{equation}
R(p_m,\gamma_t)\le \mathsf U_t(m,\gamma_t),
\label{eq:ucb-template}
\end{equation}
for every $m\in[K+1]$ and $t\in[T]$.
Thus, $\mathsf U_t(m,\gamma_t)$ is an upper confidence bound on the
expected revenue at price $p_m$ under the current inventory. The algorithm simply
selects
\(
m_t\in\argmax_{m\in[K+1]}\mathsf U_t(m,\gamma_t),
\)
posts price $p_t=p_{m_t}$, and updates its estimates after seeing the sale
$X_t$. 

Let
$m_t^\star\in\arg\max_{m\in[K+1]}R(p_m,\gamma_t)$
denote the index of a best grid price for the current inventory. On $\mathcal E$,
the finite-grid regret in round $t$ satisfies
\begin{align}
&R(p_{m_t^\star},\gamma_t)-R(p_t,\gamma_t)\nonumber\\
&=
\bigl[R(p_{m_t^\star},\gamma_t)
      -\mathsf U_t(m_t^\star,\gamma_t)\bigr]
+
\bigl[\mathsf U_t(m_t^\star,\gamma_t)
      -\mathsf U_t(m_t,\gamma_t)\bigr]
+
\bigl[\mathsf U_t(m_t,\gamma_t)-R(p_t,\gamma_t)\bigr]
\nonumber\\
&\le
\mathsf U_t(m_t,\gamma_t)-R(p_t,\gamma_t),
\label{eq:ucb-selected-excess}
\end{align}
where the inequality holds because the first bracketed term is nonpositive by \pref{eq:ucb-template}, and the second
is nonpositive because $m_t$ maximizes the index. It therefore remains to
bound the overestimation error at the selected price. For
nonnegative widths satisfying
\(
\mathsf U_t(m,\gamma_t)-R(p_m,\gamma_t)
\le w_t(m,\gamma_t)
\)
on $\mathcal E$, the finite-grid regret in round $t$ is thus at most
$w_t(m_t,\gamma_t)$. Moreover, on every sample path, we have
$0\le\Reg_T\le PDT$, since
$\gamma_t\le D$ and revenue is at most $P\gamma_t$ in each round.
The failure event therefore contributes at most
$PDT\Pr(\mathcal E^c)\le PD$ to expected regret. Combining this observation
with the price-grid approximation gives
\begin{equation}
\E[\Reg_T] \leq \frac{PDT}{K} + \E[\Reg_T^{\calP_K}]
\le \frac{PDT}{K}
   +\E\left[\sum_{t=1}^T w_t(m_t,\gamma_t)\right]
   +PD.
\label{eq:ucb-regret-template}
\end{equation}
It remains to construct UCB indices satisfying \pref{eq:ucb-template} and obtain a tight bound on the expected sum of the width $w_t(m,\gamma_t)$.

\subsection{Warm-Up: Double-Grid-UCB with
\texorpdfstring{$\otil(T^{3/4})$}{T^(3/4)} expected regret}
\label{sec:double-grid}
As a warm-up, we first instantiate the preceding UCB framework with Double-Grid-UCB, which maintains a separate revenue
estimate for each price and inventory bin, as summarized in
\pref{alg:double-grid-ucb} in \pref{app:double-grid-proof}. Specifically, in addition to the price grid, we divide
the inventory interval $[0,D]$ into $J$ bins of width $\Delta\triangleq D/J$.
Define the associated grid points by $u_j\triangleq j\Delta$ for $j\in\{0,\ldots,J\}$ and the bin assignment by
$
\ell(\gamma)\triangleq\min\{J,1+\lfloor\gamma/\Delta\rfloor\}
$
for $\gamma\in[0,D]$.
The lower representative of bin $\ell\in[J]$ is $u_{\ell-1}$, so we have
$u_{\ell(\gamma)-1}\le\gamma$ and
$\gamma-u_{\ell(\gamma)-1}\le\Delta$.
After observing $\gamma_t$, the algorithm rounds it down to $u_{\ell_t-1}$, where
$\ell_t\triangleq\ell(\gamma_t)$.
Since $0\le R(p,\gamma_t)-R(p,u_{\ell_t-1})\le P\Delta$ for every $p\in[0,P]$, the finite-grid regret satisfies
\begin{equation}
\Reg_T^{\calP_K}\le \frac{PDT}{J}+\sum_{t=1}^T\left[\max_{m\in[K+1]}R(p_m,u_{\ell_t-1})-R(p_t,u_{\ell_t-1})\right].
\label{eq:double-grid-rounded-regret}
\end{equation}
Thus, after paying the inventory-rounding cost $PDT/J$, it suffices to apply the preceding UCB argument at the grid representatives.
To estimate the revenue at the representative inventory level $u_{\ell_t-1}$, we exploit the fact that $u_{\ell_t-1}\le\gamma_t$. Although demand $Y_t(p_t)$ is only observed through the censored sales $X_t=\min\{Y_t(p_t),\gamma_t\}$, censoring at $\gamma_t$ does not affect its truncation at $u_{\ell_t-1}$. Hence, we can construct the observable clipped revenue
$p_t\min\{X_t,u_{\ell_t-1}\}=p_t\min\{Y_t(p_t),u_{\ell_t-1}\}$,
where the equality follows from $u_{\ell_t-1}\le\gamma_t$. This quantity therefore provides an uncensored observation of the revenue associated with truncating demand at $u_{\ell_t-1}$.
For each price-inventory pair $(m,\ell)$, denote by $N^{\mathrm{dg}}_{t,m,\ell}$ its number
of ``visits'' (that is, event of posting price $p_m$ at rounded inventory $u_{\ell-1}$) before round $t$, and by $Q^{\mathrm{dg}}_{t,m,\ell}$ the
sum of the corresponding clipped revenues. Define
$\widehat R^{\mathrm{dg}}_{t,m,\ell}
\triangleq Q^{\mathrm{dg}}_{t,m,\ell}/\max\{1,N^{\mathrm{dg}}_{t,m,\ell}\}$
as the empirical estimate of $R(p_m,u_{\ell-1})$.
Define the confidence radius by
\begin{equation}
C^{\mathrm{dg}}_{t,m,\ell}
\triangleq PD\sqrt{
\frac{\log(2(K+1)JT^2)}{2\max\{1,N^{\mathrm{dg}}_{t,m,\ell}\}}}.
\label{eq:double-grid-radius}
\end{equation}
Then a standard Hoeffding bound shows that, with probability at least $1-1/T$,
\begin{equation}
\left|\widehat R^{\mathrm{dg}}_{t,m,\ell}-R(p_m,u_{\ell-1})\right|\le C^{\mathrm{dg}}_{t,m,\ell},
\label{eq:double-grid-confidence-main}
\end{equation}
for every $m\in[K+1]$, $\ell\in[J]$, and $t\in[T]$. Denote this confidence event by $\mathcal E$. This confidence bound allows us to construct a revenue UCB for each price $p_m$ at inventory representative $u_{\ell-1}$ as
$
\mathsf U_t^{\mathrm{dg}}(m,u_{\ell-1})
\triangleq\widehat R^{\mathrm{dg}}_{t,m,\ell}
+C^{\mathrm{dg}}_{t,m,\ell}.$
The algorithm then selects the price maximizing
$\mathsf U_t^{\mathrm{dg}}(m,u_{\ell_t-1})$ and updates only the selected
pair $(m_t,\ell_t)$ after receiving $X_t$. Moreover, \pref{eq:double-grid-confidence-main} implies that, on $\mathcal E$, we have
$
0\le\mathsf U_t^{\mathrm{dg}}(m,u_{\ell-1})-R(p_m,u_{\ell-1})\le2C^{\mathrm{dg}}_{t,m,\ell},
$
for every $m\in[K+1]$, $\ell\in[J]$, and $t\in[T]$.
Thus, we can take $w_t(m,u_{\ell-1})\triangleq2C^{\mathrm{dg}}_{t,m,\ell}$ in the preceding UCB argument. By \pref{eq:ucb-selected-excess}, the regret at the rounded inventory in round $t$ is at most $w_t(m_t,u_{\ell_t-1})$.
Summing these widths over visits to the $(K+1)J$ pairs and adding the inventory-rounding cost from \pref{eq:double-grid-rounded-regret}, together with the price-grid and failure-event terms from \pref{eq:ucb-regret-template}, yields the following guarantee.
\begin{theorem}
\label{thm:double-grid-upper}
Suppose
\pref{assum:monotone-demand} holds. \pref{alg:double-grid-ucb} with $K=J=\lceil T^{1/4}\log^{-1/4}T\rceil$ guarantees that
$\E[\Reg_T]=\order(PDT^{3/4}\log^{1/4}T)$.
\end{theorem}
The complete proof is deferred to \pref{app:double-grid-proof}. While \pref{thm:double-grid-upper} establishes the first sublinear regret guarantee for this problem, the resulting $T^{3/4}$ rate still leaves room for improvement. The loss arises from treating inventory bins separately: a finer grid reduces the inventory-rounding error $PDT/J$, but increases the estimation term by a factor of $\sqrt J$. In particular, each observation updates only a single inventory bin, even though revenues across different inventory levels are governed by the same underlying demand distribution. This suggests constructing revenue UCBs jointly across inventory levels, allowing observations to be shared rather than estimating each bin in isolation.

\subsection{Threshold-UCB with
\texorpdfstring{$\otil(T^{2/3})$}{T^(2/3)} expected regret}
\label{sec:threshold-construction}

To share observations across inventory levels, we exploit the fact that
inventory changes how demand is censored, but not the demand distribution at a
fixed price. Specifically,
since $R(p_m,\gamma)=p_m\int_0^\gamma S_{p_m}(u)\,\mathrm du$, an estimate of
the survival function $S_{p_m}$ yields revenue estimates across inventory
levels. The difficulty is that we observe only the censored sales
$X_t=\min\{Y_t(p_t),\gamma_t\}$ rather than the demand realization itself.

\paragraph{Why Kaplan--Meier does not directly apply.}
A natural candidate is the Kaplan--Meier (KM) estimator
\citep{kaplan1958nonparametric}, which is designed for right-censored data
and has also been used for inventory learning \citep{huh2011adaptive}.
In the classical setting, one observes
$W=\min\{Y,C\}$ together with the censoring indicator
$\delta=\one\{Y\le C\}$, where $Y$ is the quantity of interest and $C$ is
the censoring level. Given a collection of such observations, the KM
estimator pools information across different censoring levels through the
product-limit construction
$
\widehat S_{\mathrm{KM}}(y)
=
\prod_{x\le y}
\left(1-\frac{d(x)}{n(x)}\right),
$
where $d(x)$ counts observations with $W=x$ and $\delta=1$, and
$n(x)$ counts observations with $W\ge x$.For i.i.d. observations with independent $Y$ and $C$,
results such as \citet{bitouze1999dkm} provide uniform error bounds for the KM estimator.

However, the required i.i.d. censoring structure is not guaranteed in our setting.
Indeed, the inventory $\gamma_t$ is observed before the price $p_t$ is selected, and
both decisions may depend on past sales. Thus, the rule selecting censoring
levels for a fixed price can change over time, and its retained observations
need not form an i.i.d. sample. Current demand still remains independent of
the pre-demand information, but this does not ensure an i.i.d. censoring history.
Therefore, the classical KM confidence guarantee does not directly apply.
We also evaluate a UCB-style baseline based on the KM estimator in
\pref{sec:experiments}, which indeed incurs
substantially higher regret than our algorithm empirically.

\paragraph{Our algorithm: Threshold-UCB.}
We are now ready to present our key idea.
As in \pref{alg:double-grid-ucb}, we still keep the price grid $\calP_K$ and optimistic price selection. However, instead of estimating revenue separately for each inventory bin, the key idea is to approximate $R(p_m,\gamma)=p_m\int_0^\gamma S_{p_m}(u)\,\mathrm du$ by a sum of survival probabilities, whose estimates can be shared across inventory levels at each price. Specifically, set $\Delta=D/J$ and $u_j\triangleq j\Delta$ for $j\in\{0,1,\dots,J\}$, where $J\ge2$, and define the discretized revenue as
\begin{equation}
R_m^{(J)}(\gamma)\triangleq p_m\Delta\sum_{j\in[J]:u_{j-1}<\gamma}S_{p_m}(u_{j-1}).
\label{eq:threshold-discrete-revenue}
\end{equation}
We then estimate each $S_{p_m}(u_{j-1})=1-F_{p_m}(u_{j-1})$ by $1-\widehat F_{t,m,j}$, where $\widehat F_{t,m,j}$ estimates the CDF value $F_{p_m}(\cdot)$ at threshold $u_{j-1}$. These probabilities depend on the price and threshold, but not on inventory. Therefore, observations collected at different inventories can estimate the same probability, and changing $\gamma$ only changes which estimates enter the revenue sum in \pref{eq:threshold-discrete-revenue}.

To construct $\widehat F_{t,m,j}$, we use the fact that whenever $u_{j-1}<\gamma_t$, the sale reveals the indicator $\one\{Y_t(p_t)\le u_{j-1}\}$ exactly as $\one\{X_t\le u_{j-1}\}$. Denote by $N_{t,m,j}$ the number of earlier rounds $s< t$ with posted price index $m_s=m$ and inventory $\gamma_s>u_{j-1}$, and by $Z_{t,m,j}$ the number of those rounds with sales $X_s\le u_{j-1}$ (see the updates in Line~\ref{algline:threshold-count-Z} of \pref{alg:threshold-ucb}). We then define $\widehat F_{t,m,j}\triangleq Z_{t,m,j}/\max\{1,N_{t,m,j}\}$. Therefore, unlike \pref{alg:double-grid-ucb}'s revenue estimate for a single bin, this estimate pools all observations at price $p_m$ whose inventories exceed $u_{j-1}$.

To see how accurate this estimate is, the same Hoeffding argument gives $|\widehat F_{t,m,j}-F_{p_m}(u_{j-1})|\le C^{\mathrm{th}}_{t,m,j}$ simultaneously for all $m\in[K+1]$, $j\in[J]$, and $t\in[T]$, with probability at least $1-1/T$, where $C^{\mathrm{th}}_{t,m,j}$ is initialized and updated in Lines~\ref{algline:threshold-radius-init} and~\ref{algline:threshold-radius-update} of \pref{alg:threshold-ucb}. For a fixed $(m,j)$, whether round $s$ contributes to the estimate depends only on the conditions $m_s=m$ and $\gamma_s>u_{j-1}$, both determined before current demand is realized. Each retained indicator therefore has conditional mean $F_{p_m}(u_{j-1})$, which is sufficient for the Hoeffding argument under adaptive sampling. In contrast, the classical KM confidence bound requires i.i.d. demand-censoring pairs.

With these estimators, we construct a UCB for the discretized revenue $R_m^{(J)}(\gamma)$ by combining the survival estimates $1-\widehat F_{t,m,j}$ with their confidence radii $C^{\mathrm{th}}_{t,m,j}$. Specifically, we define $\mathsf U_t^{\mathrm{th}}(m,\gamma)\triangleq p_m\Delta\sum_{j\in[J]:u_{j-1}<\gamma}(1-\widehat F_{t,m,j})+B_{t,m}(\gamma)$, where $B_{t,m}(\gamma)\triangleq p_m\Delta\sum_{j\in[J]:u_{j-1}<\gamma}C^{\mathrm{th}}_{t,m,j}$ and empty sums equal zero. On this uniform confidence event, the preceding concentration bound gives $0\le\mathsf U_t^{\mathrm{th}}(m,\gamma)-R_m^{(J)}(\gamma)\le2B_{t,m}(\gamma)$, so we can take $w_t(m,\gamma)\triangleq2B_{t,m}(\gamma)$ in the UCB argument for the discretized revenues. It therefore remains to bound the sum of the selected bonuses.

To bound this sum, we exploit the pooling of observations across inventory levels. For a fixed price $p$ and threshold $\gamma$, observations from different inventory bins estimate the same CDF value and therefore increase a single shared count. The accumulated confidence radii then scale with the square root of the pooled observation count, rather than the sum of separate square-root costs across bins. This avoids the additional uncertainty cost incurred by the separate inventory-bin estimates in \pref{alg:double-grid-ucb}. For each threshold, only $K+1$ pooled estimates remain, and the same confidence-radius summation gives a total contribution of $\order(\sqrt{KT\log(KJT)})$.

Moreover, every threshold contributing to the current bonus receives an observation on that round, so each contribution can be charged to an update of its shared count. Summing over the $J$ thresholds introduces a factor of $J$, which is offset by their integration weight $\Delta=D/J$. Therefore, we obtain
$
\sum_{t=1}^T B_{t,m_t}(\gamma_t)
=
\order\left(P\Delta J\sqrt{KT\log(KJT)}\right)
=
\order\left(PD\sqrt{KT\log(KJT)}\right).
$
The remaining dependence on $J$ in this estimation bound comes only from the logarithmic confidence adjustment needed to cover all prices, thresholds, and rounds simultaneously. Thus, refining the threshold grid reduces the integration error without introducing an additional $\sqrt J$ estimation factor. Combining this bound with the selected widths, the price-grid error $PDT/K$, the integration error $PDT/J$, and the failure-event contribution $PD$ yields the following guarantee.

\begingroup
\LinesNumbered
\begin{algorithm2e}[t]
\caption{Threshold-UCB}
\label{alg:threshold-ucb}
\KwIn{Price grid $\calP_K=\{p_m=(m-1)P/K:m\in[K+1]\}$, maximum inventory $D$, integer $J\ge2$, and horizon $T$}
Set $\Delta=D/J$ and $u_j=j\Delta$ for $j\in\{0,1,\dots,J\}$\;

For all $(m,j)\in[K+1]\times[J]$, initialize
$N_{1,m,j}=Z_{1,m,j}=\widehat F_{1,m,j}=0$\;
Set $C^{\mathrm{th}}_{1,m,j}=\sqrt{\log(2(K+1)JT^2)/2}$ for all $(m,j)\in[K+1]\times[J]$\nllabel{algline:threshold-radius-init}\;

\For{$t=1,\ldots,T$}{
Observe $\gamma_t$ and choose
$
m_t\in\arg\max_{m\in[K+1]}
p_m\Delta\sum_{j\in[J]:u_{j-1}<\gamma_t}
\big(1-\widehat F_{t,m,j}+C^{\mathrm{th}}_{t,m,j}\big)
$\;

Decide $p_t=p_{m_t}$ and observe sales $X_t=\min\{Y_t(p_t),\gamma_t\}$.

\For{$j\in[J]$ such that $u_{j-1}<\gamma_t$}{
Set $N_{t+1,m_t,j}=N_{t,m_t,j}+1$\nllabel{algline:threshold-count-N} and 
Set $Z_{t+1,m_t,j}=Z_{t,m_t,j}+\one\{X_t\le u_{j-1}\}$\nllabel{algline:threshold-count-Z}\;

Set
$\widehat F_{t+1,m_t,j}
=Z_{t+1,m_t,j}/N_{t+1,m_t,j}$ and $C^{\mathrm{th}}_{t+1,m_t,j}
=\sqrt{\frac{\log(2(K+1)JT^2)}{2N_{t+1,m_t,j}}}$\nllabel{algline:threshold-radius-update}\;
}

\For{$(m,j)\in[K+1]\times[J]$ with $m\ne m_t$ or $u_{j-1}\ge\gamma_t$}{
Set
$N_{t+1,m,j}=N_{t,m,j}$,
$Z_{t+1,m,j}=Z_{t,m,j}$,
$\widehat F_{t+1,m,j}=\widehat F_{t,m,j}$, and
$C^{\mathrm{th}}_{t+1,m,j}=C^{\mathrm{th}}_{t,m,j}$\;
}
}
\end{algorithm2e}
\endgroup

\begin{theorem}
\label{thm:regret-upper}
Suppose
\pref{assum:monotone-demand} holds. \pref{alg:threshold-ucb} with $K=J=\lceil T^{1/3}\log^{-1/3}T\rceil$ guarantees that $\E[\Reg_T]=\order(PDT^{2/3}\log^{1/3}T)$.
\end{theorem}

The complete proof is deferred to \pref{app:threshold-proof}. Threshold-UCB therefore achieves $\otil(PDT^{2/3})$ expected regret by sharing observations across inventory levels. To complement this guarantee, we also establish an $\Omega(T^{2/3})$ lower bound with $P=D=1$ through a reduction to classical stochastic posted pricing: constant inventory and binary demand recover the random-valuation model of \citet{kleinberg2003value}. We defer the lower-bound argument to \pref{app:lower}. This reduction also explains why the parametric rate of \citet{xu2026dynamic} does not extend to the present class. Their $\Omega(\sqrt T)$ benchmark concerns a linear demand curve with shared additive noise, whereas in the model of \citet{kleinberg2003value}, the randomness is generated by an arbitrary valuation distribution, so the resulting binary sales noise is price-dependent. Together, these bounds establish that Threshold-UCB is minimax-optimal up to logarithmic factors.
\section{Experiments}
\label{sec:experiments}
In this section, we compare Threshold-UCB (\pref{alg:threshold-ucb}) with existing benchmarks across different demand, inventory, and noise specifications, examining both the setting covered by C20CB~\citep{xu2026dynamic} and settings that depart from its assumptions. These parameter specifications are not tailored to any method. They are chosen to cover both C20CB's reference setting and departures from it, and our conclusions are consistent across all $48$ resulting environments (see \pref{app:complete-additive-results} and \pref{app:multiplicative}).

\paragraph{Demand.} We consider four baseline demand curves $\mu(p)$ for $p\in[0,1]$: (A) the linear curve $1.4-0.8p$ with a pronounced negative slope; (B) the nearly flat linear curve $1.10-0.01p$; (C) the exponential curve $\exp(0.2-p)$; and (D) the logistic curve $[1+\exp(-(2-4p))]^{-1}$. Demand A is the reference case satisfying C20CB's linearity assumption, Demand B is linear but too flat for C20CB's inventory conditions to hold, and Demands C and D are nonlinear.

\paragraph{Inventory.} We consider three support patterns: (1) central inventories in $[0.86,1.29]$; (2) broad inventories in $[0.02,1.75]$; and (3) piecewise-changing inventories that cycle among low, central, and high regimes. Each pattern uses either a uniform distribution (\texttt{U}) or a clipped-Gaussian distribution (\texttt{G}). The piecewise processes simulate abrupt changes in inventory availability, as permitted by our adversarial inventory model. The precise distributions are given in \pref{app:additive-assumption-map}.

\paragraph{Noise.}  The additive model (\texttt{Add}) uses $Y_t(p)=\mu(p)+\varepsilon_t$, where $\varepsilon_t$ is zero-mean, price-independent uniform noise with half-width $0.10$ for the linear curves, $0.05$ for the exponential curve, and $0.03$ for the logistic curve, so that demand is always nonnegative. The multiplicative model (\texttt{Mult}) uses $Y_t(p)=\mu(p)Z_t(p)$, where the multiplier has mean $1$ but a price-dependent distribution, preserving the conditional mean while changing the noise law with price. The multiplicative distributions are specified in \pref{app:multiplicative}.

Each figure title identifies the demand curve, inventory support, inventory distribution, and noise using these codes. For example, \texttt{A-1-U-Add} denotes Demand A with central uniform inventory and additive noise. To compare performance within and beyond C20CB's assumptions, we consider six representative combinations in \pref{fig:diagnostic-tuples}. Among them, only \texttt{A-1-U-Add} satisfies both its linear-additive demand and inventory-support conditions. Starting from this, \texttt{A-2-U-Add} expands the inventory support beyond the range allowed by C20CB. In \texttt{B-1-U-Add}, the demand curve is nearly flat relative to the noise, so no inventory level can produce the censoring pattern required by C20CB (that is, demand must be fully censored at the lowest price and fully observed at the highest price). The nonlinear demand curves in \texttt{C-1-U-Add} and \texttt{C-3-G-Add} violate the linear-demand assumption, while \texttt{A-1-U-Mult} retains the reference linear demand curve but replaces the additive noise with multiplicative one. Detailed assumption checks are provided in \pref{app:additive-assumption-map} and \pref{app:multiplicative}.

\paragraph{Experimental setup.} We compare tuned Threshold-UCB with the original C20CB \citep{xu2026dynamic} that uses a fixed horizon and a variant that uses a doubling trick on the horizon when the mean demand is linear, and with KM-UCB and Naive-UCB in all environments. KM-UCB uses Kaplan--Meier to estimate the demand distribution, whereas Naive-UCB treats censored sales as demand. We omit Double-Grid-UCB, which serves only as a theoretical warm-up in \pref{sec:double-grid}.
See \pref{app:benchmark} for detailed descriptions of all these benchmark algorithms. Each run lasts $T=10000$ rounds, and regret is evaluated against an oracle on a $2001$-point price grid, much finer than the learners' price grids, which have at most $45$ points after tuning. 
 Curves average $30$ held-out paths. In each environment, each method receives $20$ hyperparameter configurations evaluated on $5$ separate tuning paths. Implementation details are given in \pref{app:experiment-details}.

\begin{figure}[t]
\centering
\begin{minipage}[t]{0.30\textwidth}
\centering
\includegraphics[width=\linewidth]{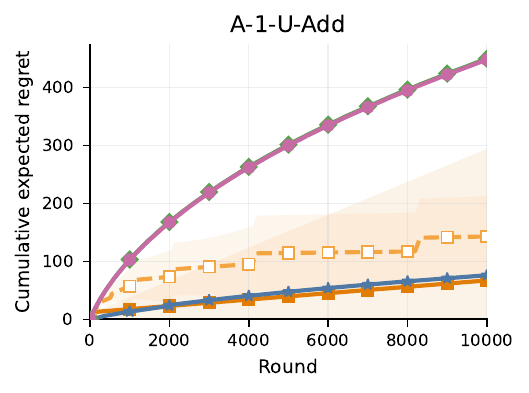}
\end{minipage}\hfill
\begin{minipage}[t]{0.30\textwidth}
\centering
\includegraphics[width=\linewidth]{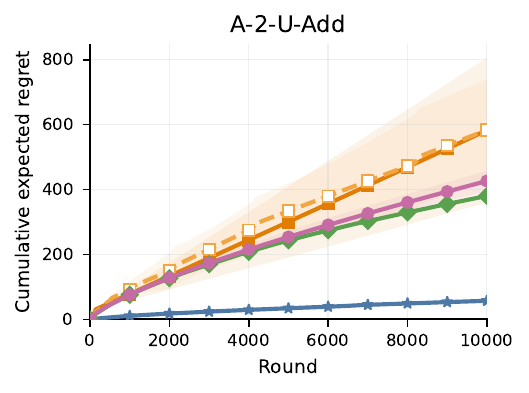}
\end{minipage}\hfill
\begin{minipage}[t]{0.30\textwidth}
\centering
\includegraphics[width=\linewidth]{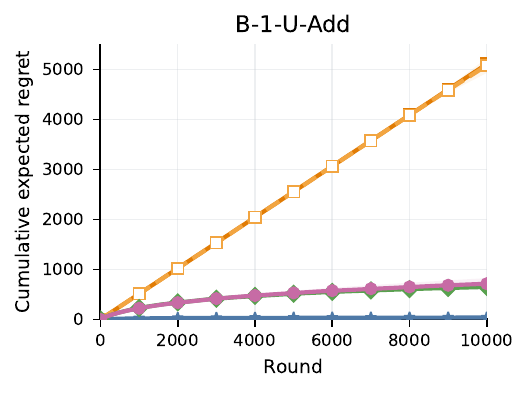}
\end{minipage}

\vspace{0.6em}
\begin{minipage}[t]{0.30\textwidth}
\centering
\includegraphics[width=\linewidth]{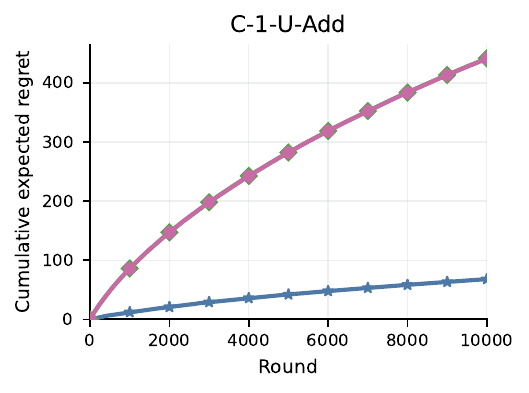}
\end{minipage}\hfill
\begin{minipage}[t]{0.30\textwidth}
\centering
\includegraphics[width=\linewidth]{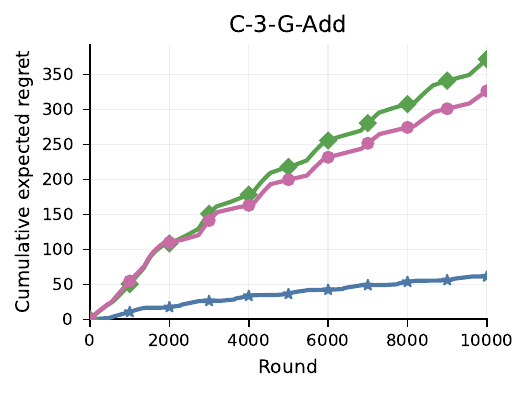}
\end{minipage}\hfill
\begin{minipage}[t]{0.30\textwidth}
\centering
\includegraphics[width=\linewidth]{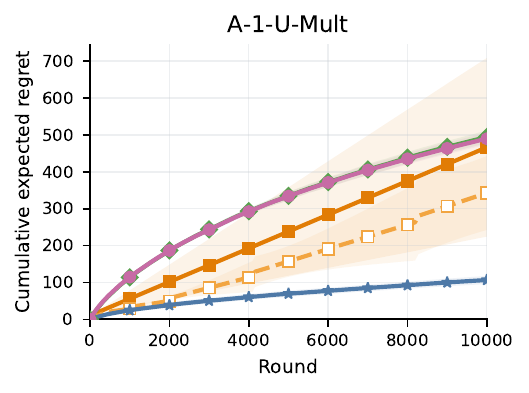}
\end{minipage}
\includegraphics[width=\linewidth]{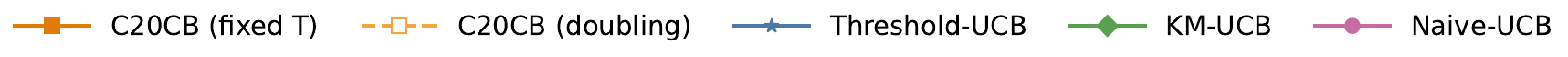}\caption{Mean cumulative expected regret over 30 held-out paths; shading indicates one sample standard deviation. Titles specify demand, inventory pattern, and noise type. Only \texttt{A-1-U-Add} matches C20CB's conditions. Broad inventory violates its support conditions, nearly flat Demand B makes the admissible inventory interval empty, and multiplicative noise violates its common additive-noise assumption. C20CB is omitted for nonlinear Demand C. Threshold-UCB achieves the lowest final regret in all panels except \texttt{A-1-U-Add}, where fixed-horizon C20CB is slightly lower.} \label{fig:diagnostic-tuples}
\end{figure}

\paragraph{Results.}
Threshold-UCB achieves the lowest or nearly lowest mean final regret among the compared methods in all six environments in \pref{fig:diagnostic-tuples}. In the matched reference environment \texttt{A-1-U-Add}, fixed-horizon C20CB performs best, while Threshold-UCB attains comparable mean regret. Thus, our nonparametric method remains competitive when C20CB's stronger structural conditions hold.

Outside this reference setting, Threshold-UCB has the lowest final regret in each displayed environment. Under linear demand and additive noise, it outperforms all benchmarks in \texttt{A-2-U-Add} and \texttt{B-1-U-Add}, where C20CB's inventory conditions fail. Under nonlinear demand where C20CB does not apply, Threshold-UCB outperforms the heuristic KM-UCB and Naive-UCB baselines in both \texttt{C-1-U-Add} and \texttt{C-3-G-Add}, including the piecewise-inventory setting. Finally, in the case of \texttt{A-1-U-Mult} with price-dependent multiplicative noise, Threshold-UCB again achieves the lowest mean final regret, while both C20CB variants incur substantially higher regret. Additional results across the remaining demand, inventory, and noise combinations are reported in \pref{app:experiment-details}, \pref{app:complete-additive-results}, and \pref{app:multiplicative}. Across all $48$ environments, Threshold-UCB attains the lowest mean final regret in $46$, the exceptions being \texttt{A-1-U-Add} and \texttt{A-1-G-Add}, where C20CB's conditions hold.
\bibliographystyle{plainnat}
\bibliography{references}

@article{xu2021logarithmic,
  title={Logarithmic regret in feature-based dynamic pricing},
  author={Xu, Jianyu and Wang, Yu-Xiang},
  journal={Advances in Neural Information Processing Systems},
  volume={34},
  pages={13898--13910},
  year={2021}
}

@article{luo2022contextual,
  title={Contextual dynamic pricing with unknown noise: Explore-then-ucb strategy and improved regrets},
  author={Luo, Yiyun and Sun, Will Wei and Liu, Yufeng},
  journal={Advances in Neural Information Processing Systems},
  volume={35},
  pages={37445--37457},
  year={2022}
}

@inproceedings{xu2022towards,
  title={Towards agnostic feature-based dynamic pricing: Linear policies vs linear valuation with unknown noise},
  author={Xu, Jianyu and Wang, Yu-Xiang},
  booktitle={International Conference on Artificial Intelligence and Statistics},
  pages={9643--9662},
  year={2022},
  organization={PMLR}
}

@article{javanmard2019dynamic,
  title={Dynamic pricing in high-dimensions},
  author={Javanmard, Adel and Nazerzadeh, Hamid},
  journal={Journal of Machine Learning Research},
  volume={20},
  number={9},
  pages={1--49},
  year={2019}
}

@article{kaplan1958nonparametric,
  title={Nonparametric estimation from incomplete observations},
  author={Kaplan, Edward L and Meier, Paul},
  journal={Journal of the American statistical association},
  volume={53},
  number={282},
  pages={457--481},
  year={1958},
  publisher={Taylor \& Francis}
}

@article{hoeffding1963probability,
  author  = {Hoeffding, Wassily},
  title   = {Probability Inequalities for Sums of Bounded Random Variables},
  journal = {Journal of the American Statistical Association},
  volume  = {58},
  number  = {301},
  pages   = {13--30},
  year    = {1963}
}

@article{gallego1994optimal,
  author  = {Gallego, Guillermo and van Ryzin, Garrett},
  title   = {Optimal Dynamic Pricing of Inventories with Stochastic Demand over Finite Horizons},
  journal = {Management Science},
  volume  = {40},
  number  = {8},
  pages   = {999--1020},
  year    = {1994},
  doi     = {10.1287/mnsc.40.8.999}
}

@inproceedings{kleinberg2003value,
  author    = {Kleinberg, Robert D. and Leighton, F. Thomson},
  title     = {The Value of Knowing a Demand Curve: Bounds on Regret for Online Posted-Price Auctions},
  booktitle = {Proceedings of the 44th Annual IEEE Symposium on Foundations of Computer Science},
  pages     = {594--605},
  year      = {2003},
  doi       = {10.1109/SFCS.2003.1238232}
}

@article{besbes2009dynamic,
  author  = {Besbes, Omar and Zeevi, Assaf},
  title   = {Dynamic Pricing Without Knowing the Demand Function: Risk Bounds and Near-Optimal Algorithms},
  journal = {Operations Research},
  volume  = {57},
  number  = {6},
  pages   = {1407--1420},
  year    = {2009},
  doi     = {10.1287/opre.1080.0640}
}

@article{broder2012dynamic,
  author  = {Broder, Josef and Rusmevichientong, Paat},
  title   = {Dynamic Pricing Under a General Parametric Choice Model},
  journal = {Operations Research},
  volume  = {60},
  number  = {4},
  pages   = {965--980},
  year    = {2012},
  doi     = {10.1287/opre.1120.1057}
}

@article{keskin2014dynamic,
  author  = {Keskin, N. Bora and Zeevi, Assaf},
  title   = {Dynamic Pricing with an Unknown Demand Model: Asymptotically Optimal Semi-Myopic Policies},
  journal = {Operations Research},
  volume  = {62},
  number  = {5},
  pages   = {1142--1167},
  year    = {2014},
  doi     = {10.1287/opre.2014.1294}
}

@article{huh2011adaptive,
  author  = {Huh, Woonghee Tim and Levi, Retsef and Rusmevichientong, Paat and Orlin, James B.},
  title   = {Adaptive Data-Driven Inventory Control with Censored Demand Based on Kaplan--Meier Estimator},
  journal = {Operations Research},
  volume  = {59},
  number  = {4},
  pages   = {929--941},
  year    = {2011},
  doi     = {10.1287/opre.1100.0906}
}

@article{bisi2011censored,
  author  = {Bisi, Arnab and Dada, Maqbool and Tokdar, Surya},
  title   = {A Censored-Data Multiperiod Inventory Problem with Newsvendor Demand Distributions},
  journal = {Manufacturing \& Service Operations Management},
  volume  = {13},
  number  = {4},
  pages   = {525--533},
  year    = {2011},
  doi     = {10.1287/msom.1110.0340}
}

@article{xu2026dynamic,
  author  = {Xu, Jianyu and Wang, Yining and Chen, Xi and Wang, Yu-Xiang},
  title   = {Dynamic Pricing with Adversarially-Censored Demands},
  journal = {arXiv preprint arXiv:2502.06168v2},
  year    = {2026},
  doi     = {10.48550/arXiv.2502.06168}
}

@article{bitouze1999dkm,
  author  = {Bitouz{\'e}, Denis and Laurent, B{\'e}atrice and Massart, Pascal},
  title   = {A Dvoretzky--Kiefer--Wolfowitz Type Inequality for the Kaplan--Meier Estimator},
  journal = {Annales de l'Institut Henri Poincar{\'e}, Probabilit{\'e}s et Statistiques},
  volume  = {35},
  number  = {6},
  pages   = {735--763},
  year    = {1999}
}

@article{chen2010bounds,
  author  = {Chen, Li},
  title   = {Bounds and Heuristics for Optimal {Bayesian} Inventory Control with Unobserved Lost Sales},
  journal = {Operations Research},
  volume  = {58},
  number  = {2},
  pages   = {396--413},
  year    = {2010},
  doi     = {10.1287/opre.1090.0726}
}

@article{besbes2013implications,
  author  = {Besbes, Omar and Muharremoglu, Alp},
  title   = {On Implications of Demand Censoring in the Newsvendor Problem},
  journal = {Management Science},
  volume  = {59},
  number  = {6},
  pages   = {1407--1424},
  year    = {2013},
  doi     = {10.1287/mnsc.1120.1654}
}

@article{huh2009adaptive,
  author  = {Huh, Woonghee Tim and Janakiraman, Ganesh and Muckstadt, John A. and Rusmevichientong, Paat},
  title   = {An Adaptive Algorithm for Finding the Optimal Base-Stock Policy in Lost Sales Inventory Systems with Censored Demand},
  journal = {Mathematics of Operations Research},
  volume  = {34},
  number  = {2},
  pages   = {397--416},
  year    = {2009},
  doi     = {10.1287/moor.1080.0367}
}

@article{chen2024optimal,
  author  = {Chen, Boxiao and Wang, Yining and Zhou, Yuan},
  title   = {Optimal Policies for Dynamic Pricing and Inventory Control with Nonparametric Censored Demands},
  journal = {Management Science},
  volume  = {70},
  number  = {5},
  pages   = {3362--3380},
  year    = {2024},
  doi     = {10.1287/mnsc.2023.4859}
}

@article{chen2021nonparametric,
  author  = {Chen, Boxiao and Chao, Xiuli and Shi, Cong},
  title   = {Nonparametric Learning Algorithms for Joint Pricing and Inventory Control with Lost Sales and Censored Demand},
  journal = {Mathematics of Operations Research},
  volume  = {46},
  number  = {2},
  pages   = {726--756},
  year    = {2021},
  doi     = {10.1287/moor.2020.1084}
}

@article{chen2020databased,
  author  = {Chen, Boxiao and Chao, Xiuli and Wang, Yining},
  title   = {Data-Based Dynamic Pricing and Inventory Control with Censored Demand and Limited Price Changes},
  journal = {Operations Research},
  volume  = {68},
  number  = {5},
  pages   = {1445--1456},
  year    = {2020},
  doi     = {10.1287/opre.2020.1993}
}

\clearpage
\appendix
\raggedbottom
\section{Detailed Comparison with~\citet{xu2026dynamic}}\label{app:xu-comparison}
In this section, we elaborate on the discussion in \pref{sec:comparison} and show that the conditions in \pref{eq:linear-inventory-support} not only are restrictive but also greatly simplify learning: the known strict gap yields a high-price
interval on which demand is never censored.  Specifically, define
$
\epsilon
\triangleq\min\left\{\frac{P}{2},
\frac{\gamma_{\min}-a_{\max}+b_{\min}P-c}{2b_{\max}}\right\}>0.
$
Then, for every $q\in[P-\epsilon,P]$, we have
$
Y_t(q)
\le a_{\max}-b_{\min}P+c+b_{\max}\epsilon
<\gamma_{\min}\le\gamma_t.
$
In particular, $P$ and $P-\epsilon$ are admissible prices at which sales
reveal demand exactly for every admissible inventory sequence.  Since
$\E[N_t]=0$, we know that
$
b
=\frac{\E\left[Y_t(P-\epsilon)\right]
-\E\left[Y_t(P)\right]}{\epsilon}$ and $
a=\E\left[Y_t(P)\right]+bP$.
Repeated observations at these two fixed prices therefore allow $a$ and $b$
to be estimated by sample means, without using censored observations.
Moreover, for either price $q$, $Y_t(q)-\E[Y_t(q)]=N_t$; hence, centering the
uncensored observations also allows us to estimate the common noise
distribution, with the centering error controlled by the sample-mean error.

In addition, the same structure supports direct revenue estimation across prices and
inventory levels.  An uncensored observation collected at $P$ satisfies
$X_t=a-bP+N_t$.  For any candidate $p\in[0,P]$ and inventory
$\gamma$, we have the identity
\begin{align*}
p\min\{X_t+b(P-p),\gamma\}
&=p\min\{a-bp+N_t,\gamma\},
\end{align*}
which shows that the transformed observation has expectation $R(p,\gamma)$.
Replacing $b$ by an estimate $\widehat b$ changes the transformed revenue by
at most $p(P-p)|\widehat b-b|$.  Combining this error bound with a confidence
band for the uncensored demand distribution gives a direct revenue-UCB
construction. Further combining uniform derivative estimates with the local curvature of the revenue function yields the $\otil(\sqrt T)$ expected regret of C20CB.
Thus, these support conditions make both parameter estimation and
distribution learning accessible through uncensored exploration. Our model, by contrast, assumes only monotonicity of expected sales.
It allows nonlinear, heteroskedastic, and price-dependent demand laws.
The $\Omega(T^{2/3})$ lower bound in
\pref{app:lower} establishes the statistical cost of this broader class, even
when inventory is constant and censoring disappears.
\section{Omitted Details in \pref{sec:algorithm}}
\label{app:upper-proof}
In this section, we provide the omitted details for \pref{sec:algorithm}. First, we prove the $\otil(T^{3/4})$ expected regret guarantee for Double-Grid-UCB presented in \pref{sec:double-grid}. We
then prove the two lemmas used in the Threshold-UCB analysis and combine them to obtain its $\otil(T^{2/3})$ expected regret guarantee in \pref{sec:threshold-construction}.

\subsection{Omitted Details in \pref{sec:double-grid}}
\label{app:double-grid-proof}

In this section, we provide the omitted details for \pref{sec:double-grid}. We first give the complete pseudocode for Double-Grid-UCB in \pref{alg:double-grid-ucb} and then prove \pref{thm:double-grid-upper}.

\subsubsection{Algorithm Description of Double-Grid-UCB}

\begin{algorithm2e}[H]
\caption{Double-Grid-UCB}
\label{alg:double-grid-ucb}
\KwIn{Price bound $P$, inventory bound $D$, integers $K,J\ge1$, and horizon $T$}
Set $p_m\triangleq(m-1)P/K$ for $m\in[K+1]$, $\Delta\triangleq D/J$, and
$u_j\triangleq j\Delta$ for $j\in\{0,\ldots,J\}$\;

For all $(m,\ell)\in[K+1]\times[J]$, initialize
$N^{\mathrm{dg}}_{1,m,\ell}=Q^{\mathrm{dg}}_{1,m,\ell}
=\widehat R^{\mathrm{dg}}_{1,m,\ell}=0$ and
$C^{\mathrm{dg}}_{1,m,\ell}
=PD\sqrt{\log(2(K+1)JT^2)/2}$\;

\For{$t=1,\ldots,T$}{
Observe $\gamma_t$ and set
$\ell_t\triangleq\min\{J,1+\lfloor\gamma_t/\Delta\rfloor\}$\;

Choose
$m_t\in\arg\max_{m\in[K+1]}
\{\widehat R^{\mathrm{dg}}_{t,m,\ell_t}
+C^{\mathrm{dg}}_{t,m,\ell_t}\}$\;

Post $p_t=p_{m_t}$ and observe $X_t$\;

Set
$N^{\mathrm{dg}}_{t+1,m_t,\ell_t}
=N^{\mathrm{dg}}_{t,m_t,\ell_t}+1$
and
$Q^{\mathrm{dg}}_{t+1,m_t,\ell_t}
=Q^{\mathrm{dg}}_{t,m_t,\ell_t}
+p_t\min\{X_t,u_{\ell_t-1}\}$\;

Set
$\widehat R^{\mathrm{dg}}_{t+1,m_t,\ell_t}
=Q^{\mathrm{dg}}_{t+1,m_t,\ell_t}/N^{\mathrm{dg}}_{t+1,m_t,\ell_t}$
and
$C^{\mathrm{dg}}_{t+1,m_t,\ell_t}
=PD\sqrt{\frac{\log(2(K+1)JT^2)}{2N^{\mathrm{dg}}_{t+1,m_t,\ell_t}}}$\;

\For{$(m,\ell)\in[K+1]\times[J]$ with $(m,\ell)\ne(m_t,\ell_t)$}{
set
$N^{\mathrm{dg}}_{t+1,m,\ell}=N^{\mathrm{dg}}_{t,m,\ell}$,
$Q^{\mathrm{dg}}_{t+1,m,\ell}=Q^{\mathrm{dg}}_{t,m,\ell}$,
$\widehat R^{\mathrm{dg}}_{t+1,m,\ell}
=\widehat R^{\mathrm{dg}}_{t,m,\ell}$, and
$C^{\mathrm{dg}}_{t+1,m,\ell}=C^{\mathrm{dg}}_{t,m,\ell}$\;
}
}
\end{algorithm2e}

\subsubsection{Proof of \pref{thm:double-grid-upper}}

\begin{proof}

By the reduction in \pref{eq:regret-decomposition} and the inventory-rounding bound in \pref{eq:double-grid-rounded-regret}, we have
\begin{align}
\Reg_T
&\le \frac{PDT}{K}+\Reg_T^{\calP_K}
\tag{\pref{eq:regret-decomposition}}\\
&\le \frac{PDT}{K}+\frac{PDT}{J}
+\sum_{t=1}^T\left[\max_{m\in[K+1]}R(p_m,u_{\ell_t-1})-R(p_t,u_{\ell_t-1})\right].
\label{eq:double-grid-inventory-error}
\end{align}
It therefore suffices to bound regret at the lower inventory representatives. We use the confidence event from \pref{eq:double-grid-confidence-main}: denote by $\mathcal E^{\mathrm{dg}}$ the event that
\begin{equation}
\left|\widehat R^{\mathrm{dg}}_{t,m,\ell}-R(p_m,u_{\ell-1})\right|
\le C^{\mathrm{dg}}_{t,m,\ell}
\label{eq:appendix-double-grid-confidence}
\end{equation}
holds for every $t\in[T+1]$, $m\in[K+1]$, and $\ell\in[J]$. \pref{lem:double-grid-confidence} shows that $\mathcal E^{\mathrm{dg}}$ holds with probability at least $1-\frac{1}{T}$.

Suppose that $\mathcal E^{\mathrm{dg}}$ holds. Applying the definition of $m_t$, we know that
\begin{align}
&\max_{m\in[K+1]}R(p_m,u_{\ell_t-1})-R(p_t,u_{\ell_t-1})\nonumber\\
&\leq \max_{m\in[K+1]}\mathsf U_t^{\mathrm{dg}}(m ,u_{\ell_t-1})-R(p_t,u_{\ell_t-1})
\tag{as $\mathcal{E}^{\mathrm{dg}}$ holds}\\
&\le \mathsf U_t^{\mathrm{dg}}(m_t,u_{\ell_t-1})-R(p_t,u_{\ell_t-1})
\tag{by definition of $m_t$}\\
&\le 2C^{\mathrm{dg}}_{t,m_t,\ell_t}.
\label{eq:double-grid-selected-width}
\end{align}
The last inequality follows from \pref{eq:appendix-double-grid-confidence}. Since $0\le\Reg_T\le PDT$ on every sample path, we obtain
\begin{align*}
\E[\Reg_T]
&=\E\left[\Reg_T\one\{\mathcal E^{\mathrm{dg}}\}\right]
+\E\left[\Reg_T\one\{(\mathcal E^{\mathrm{dg}})^c\}\right]\\
&\le \frac{PDT}{K}+\frac{PDT}{J}
+2\E\left[\sum_{t=1}^T C^{\mathrm{dg}}_{t,m_t,\ell_t}\right]+PD.
\tag{\pref{lem:double-grid-confidence}}
\end{align*}
Thus, it remains to bound the third term. For each pair $(m,\ell)$, the counts before its successive visits are $0,\ldots,N^{\mathrm{dg}}_{T+1,m,\ell}-1$. Also, $\sum_{r=0}^{n-1}(\max\{1,r\})^{-1/2}\le2\sqrt n$ for every integer $n\ge1$, with an empty sum when $n=0$. For notational convenience, let $L\triangleq \log(2(K+1)JT^2)$. Therefore, on every sample path, we know that
\begin{align}
\sum_{t=1}^T C^{\mathrm{dg}}_{t,m_t,\ell_t}&=PD\sqrt{\frac L2}
\sum_{m=1}^{K+1}\sum_{\ell=1}^J
\sum_{r=0}^{N^{\mathrm{dg}}_{T+1,m,\ell}-1}
\frac{1}{\sqrt{\max\{1,r\}}}
\nonumber\\
&\le3PD\sqrt L\sum_{m=1}^{K+1}\sum_{\ell=1}^J
\sqrt{N^{\mathrm{dg}}_{T+1,m,\ell}}
\nonumber\\
&\le3PD\sqrt{L(K+1)J\sum_{m=1}^{K+1}\sum_{\ell=1}^J N^{\mathrm{dg}}_{T+1,m,\ell}}
\tag{Cauchy-Schwarz inequality}\\
&=3PD\sqrt{(K+1)JTL}.
\label{eq:double-grid-radius-sum}
\end{align}
The last equality uses $\sum_{m,\ell}N^{\mathrm{dg}}_{T+1,m,\ell}=T$, since each round updates exactly one price-bin pair. Substituting this bound gives
\begin{equation}
\E[\Reg_T]\le\frac{PDT}{K}+\frac{PDT}{J}
+6PD\sqrt{(K+1)JTL}+PD.
\label{eq:double-grid-upper}
\end{equation}
Choosing $K=J=\lceil T^{1/4}\log^{-1/4}T\rceil$ for $T\ge2$ yields $\E[\Reg_T]=\order(PDT^{3/4}\log^{1/4}T)$.
\end{proof}

\begin{lemma}\label{lem:double-grid-confidence}
The sequence generated by \pref{alg:double-grid-ucb} satisfies $\Prob(\mathcal E^{\mathrm{dg}})\ge1-1/T$.
\end{lemma}
\begin{proof}
Fix $(m,\ell)\in[K+1]\times[J]$, and let $\tau_1<\tau_2<\cdots$ denote the successive rounds at which $(m_t,\ell_t)=(m,\ell)$. On such a round, $u_{\ell-1}\le\gamma_t$, and hence $ p_t\min\{X_t,u_{\ell-1}\} =p_m\min\{Y_t(p_m),u_{\ell-1}\}. $ Moreover, the event $\{(m_t,\ell_t)=(m,\ell)\}$ is determined before the current demand is realized. Since $Y_t(\cdot)$ is independent of $\mathcal G_t$ and has the same law in every period, the variables $ W_r\triangleq p_m\min\{Y_{\tau_r}(p_m),u_{\ell-1}\}, r\ge1,$ are i.i.d., take values in \([0,PD]\), and have mean \(R(p_m,u_{\ell-1})\).

For each fixed $n\in[T]$, Hoeffding's inequality (\pref{lem:hoeffding}) gives
\begin{align*}
&\Prob\left(\left|\frac1n\sum_{r=1}^n W_r-R(p_m,u_{\ell-1})\right|
>PD\sqrt{\frac{\log(2(K+1)JT^2)}{2n}}\right)
\le\frac{1}{(K+1)JT^2}.
\end{align*}
At zero count, the confidence bound holds by initialization.
Taking a union bound over all pairs $(m,\ell)$ and positive sample counts $n\in[T]$ yields
\begin{align*}
\Prob\bigl(\mathcal E^{\mathrm{dg}}\text{~does not hold}\bigr)
&\le\sum_{m=1}^{K+1}\sum_{\ell=1}^J\sum_{n=1}^T
\frac{1}{(K+1)JT^2}
=\frac{1}{T}.
\end{align*}
\end{proof}

\subsection{Omitted Details in \pref{sec:threshold-construction}}
\label{app:threshold-proof}
 
In this section, we provide the omitted details for \pref{sec:threshold-construction}. The pseudocode of Threshold-UCB is given in \pref{alg:threshold-ucb}. We first prove \pref{thm:regret-upper} using three supporting lemmas, and then prove these lemmas. Throughout, recall that $\Delta=D/J$, $u_j=j\Delta$ for $j\in\{0,\ldots,J\}$, and
\begin{align*}
\mathsf U_t^{\mathrm{th}}(m,\gamma)&=p_m\Delta\sum_{j\in[J]:u_{j-1}<\gamma}\left(1-\widehat F_{t,m,j}\right)+B_{t,m}(\gamma),\\
B_{t,m}(\gamma)&=p_m\Delta\sum_{j\in[J]:u_{j-1}<\gamma}C^{\mathrm{th}}_{t,m,j},
\end{align*}
with width $w_t(m,\gamma)=2B_{t,m}(\gamma)$. For notational convenience, let $L\triangleq\log(2(K+1)JT^2)$.
 
\subsubsection{Proof of \pref{thm:regret-upper}}
 
\begin{proof}
Unlike Double-Grid-UCB, Threshold-UCB constructs its indices for the discretized revenues $R_m^{(J)}$ in \pref{eq:threshold-discrete-revenue} instead of rounding the inventory. By the reduction in \pref{eq:regret-decomposition} and the integration-error bound in \pref{eq:threshold-integration-error} of \pref{lem:main-ucb-validity}, we have
\begin{align}
\Reg_T
&\le\frac{PDT}{K}+\Reg_T^{\calP_K}
\tag{\pref{eq:regret-decomposition}}\\
&\le\frac{PDT}{K}+\frac{PDT}{J}
+\sum_{t=1}^T\left[\max_{m\in[K+1]}R_m^{(J)}(\gamma_t)-R_{m_t}^{(J)}(\gamma_t)\right].
\label{eq:threshold-rounded-regret}
\end{align}
Here, only one integration error is paid in each round: every grid revenue satisfies $R(p_m,\gamma_t)\le R_m^{(J)}(\gamma_t)$, while the selected revenue satisfies $R(p_t,\gamma_t)\ge R_{m_t}^{(J)}(\gamma_t)-P\Delta$, and $TP\Delta=PDT/J$. It therefore suffices to bound regret with respect to the discretized revenues. Denote by $\mathcal E^{\mathrm{th}}$ the event that
\begin{equation}
\left|\widehat F_{t,m,j}-F_{p_m}(u_{j-1})\right|
\le C^{\mathrm{th}}_{t,m,j}
\label{eq:threshold-confidence}
\end{equation}
holds for every $t\in[T+1]$, $m\in[K+1]$, and $j\in[J]$. \pref{lem:threshold-confidence} shows that $\mathcal E^{\mathrm{th}}$ holds with probability at least $1-\frac1T$.
 
Suppose that $\mathcal E^{\mathrm{th}}$ holds. Applying the definition of $m_t$, we know that
\begin{align}
&\max_{m\in[K+1]}R_m^{(J)}(\gamma_t)-R_{m_t}^{(J)}(\gamma_t)\nonumber\\
&\le\max_{m\in[K+1]}\mathsf U_t^{\mathrm{th}}(m,\gamma_t)-R_{m_t}^{(J)}(\gamma_t)
\tag{as $\mathcal E^{\mathrm{th}}$ holds, by \pref{lem:main-ucb-validity}}\\
&\le\mathsf U_t^{\mathrm{th}}(m_t,\gamma_t)-R_{m_t}^{(J)}(\gamma_t)
\tag{by definition of $m_t$}\\
&\le w_t(m_t,\gamma_t),
\label{eq:threshold-selected-excess}
\end{align}
where the last inequality follows from \pref{eq:threshold-selected-width} in \pref{lem:main-ucb-validity}. Since $0\le\Reg_T\le PDT$ on every sample path, we obtain
\begin{align*}
\E[\Reg_T]
&=\E\left[\Reg_T\one\{\mathcal E^{\mathrm{th}}\}\right]
+\E\left[\Reg_T\one\{(\mathcal E^{\mathrm{th}})^c\}\right]\\
&\le\frac{PDT}{K}+\frac{PDT}{J}
+\E\left[\sum_{t=1}^T w_t(m_t,\gamma_t)\right]+PD.
\tag{\pref{lem:threshold-confidence}}
\end{align*}
Thus, it remains to bound the third term. Since $w_t(m_t,\gamma_t)=2B_{t,m_t}(\gamma_t)$, \pref{lem:main-bonus-control} shows that, on every sample path,
\[
\sum_{t=1}^T w_t(m_t,\gamma_t)=2\sum_{t=1}^T B_{t,m_t}(\gamma_t)\le6PD\sqrt{(K+1)TL}.
\]
Substituting this bound gives
\begin{equation}
\E[\Reg_T]\le\frac{PDT}{K}+\frac{PDT}{J}
+6PD\sqrt{(K+1)TL}+PD.
\label{eq:regret-upper}
\end{equation}
Choosing $K=J=\lceil T^{1/3}\log^{-1/3}T\rceil$ for $T\ge2$ yields $\E[\Reg_T]=\order(PDT^{2/3}\log^{1/3}T)$.
\end{proof}
 
We now prove the three supporting lemmas. The first shows that the confidence event $\mathcal E^{\mathrm{th}}$ holds with high probability.
 
\begin{lemma}\label{lem:threshold-confidence}
The sequence generated by \pref{alg:threshold-ucb} satisfies $\Prob(\mathcal E^{\mathrm{th}})\ge1-1/T$.
\end{lemma}
\begin{proof}
Fix $(m,j)\in[K+1]\times[J]$. Whenever $m_t=m$ and $\gamma_t>u_{j-1}$, the observable indicator satisfies $\one\{X_t\le u_{j-1}\}=\one\{Y_t(p_m)\le u_{j-1}\}$. Moreover, the event $\{m_t=m,\gamma_t>u_{j-1}\}$ is determined before the current demand is realized. Since $Y_t(\cdot)$ is independent of $\calG_t$ and has the same law in every period, the same argument as in the proof of \pref{lem:double-grid-confidence} shows that the retained indicators are consecutive entries $W_1,W_2,\ldots$ of an i.i.d. Bernoulli sequence with mean $F_{p_m}(u_{j-1})$.
 
For each fixed $n\in[T]$, Hoeffding's inequality (\pref{lem:hoeffding}) gives
\begin{align*}
\Prob\left(\left|\frac1n\sum_{r=1}^n W_r-F_{p_m}(u_{j-1})\right|
>\sqrt{\frac{L}{2n}}\right)
\le2e^{-L}=\frac{1}{(K+1)JT^2}.
\end{align*}
Whenever $N_{t,m,j}=n\ge1$, the estimate $\widehat F_{t,m,j}$ is the average of the first $n$ entries. At zero count, the error is at most one and the confidence bound holds by initialization, since $C^{\mathrm{th}}_{1,m,j}=\sqrt{L/2}\ge1$ when $K\ge1$ and $J\ge2$. Taking a union bound over all pairs $(m,j)$ and positive sample counts $n\in[T]$ yields
\begin{align*}
\Prob\bigl(\mathcal E^{\mathrm{th}}\text{~does not hold}\bigr)
&\le\sum_{m=1}^{K+1}\sum_{j=1}^J\sum_{n=1}^T
\frac{1}{(K+1)JT^2}
=\frac{1}{T}.
\end{align*}
\end{proof}
 
The next lemma shows that, on $\mathcal E^{\mathrm{th}}$, $\mathsf U_t^{\mathrm{th}}(m,\gamma)$ is a valid upper confidence bound on the discretized revenue $R_m^{(J)}(\gamma)$ with width at most $w_t(m,\gamma)$. It also bounds the error of approximating $R(p_m,\gamma)$ by $R_m^{(J)}(\gamma)$.
 
\begin{lemma}\label{lem:main-ucb-validity}
On $\mathcal E^{\mathrm{th}}$, for every $t\in[T]$, $m\in[K+1]$, and $\gamma\in[0,D]$, we have
\begin{equation}
0\le\mathsf U_t^{\mathrm{th}}(m,\gamma)-R_m^{(J)}(\gamma)
\le2B_{t,m}(\gamma)=w_t(m,\gamma).
\label{eq:threshold-selected-width}
\end{equation}
Moreover, for every $m\in[K+1]$ and $\gamma\in[0,D]$, we have
\begin{equation}
0\le R_m^{(J)}(\gamma)-R(p_m,\gamma)
\le p_m\Delta\le P\Delta.
\label{eq:threshold-integration-error}
\end{equation}
\end{lemma}
 
\begin{proof}
We first establish \pref{eq:threshold-selected-width}. Since $S_{p_m}(u_{j-1})=1-F_{p_m}(u_{j-1})$, the definitions of $\mathsf U_t^{\mathrm{th}}$ and $R_m^{(J)}$ give
\begin{align*}
\mathsf U_t^{\mathrm{th}}(m,\gamma)-R_m^{(J)}(\gamma)
=p_m\Delta\sum_{j\in[J]:u_{j-1}<\gamma}
\left(F_{p_m}(u_{j-1})-\widehat F_{t,m,j}+C^{\mathrm{th}}_{t,m,j}\right).
\end{align*}
On $\mathcal E^{\mathrm{th}}$, each summand in parentheses lies in $[0,2C^{\mathrm{th}}_{t,m,j}]$. Summing gives \pref{eq:threshold-selected-width} simultaneously for every $\gamma\in[0,D]$, since changing inventory only selects a subset of the same threshold estimates.
 
For the integration error, if $\gamma=0$, then both $R_m^{(J)}(\gamma)$ and $R(p_m,\gamma)$ equal zero. Otherwise, let $n\triangleq|\{j\in[J]:u_{j-1}<\gamma\}|$, so that $n\ge1$ and $u_{n-1}<\gamma\le u_n\le D$. Since $S_{p_m}(\cdot)$ is nonincreasing, for every $j\in[n]$ we have
$
\int_{u_{j-1}}^{u_j}S_{p_m}(u)\,\dd u
\le\Delta S_{p_m}(u_{j-1})$ and if $j\ge2$, $\Delta S_{p_m}(u_{j-1})
\le\int_{u_{j-2}}^{u_{j-1}}S_{p_m}(u)\,\dd u.$ Summing the first inequality over $j\in[n]$ and the second inequality over $2\leq j\leq n$, and using $[0,u_n]\supseteq[0,\gamma]$ and $[0,u_{n-1}]\subseteq[0,\gamma]$, we have
\begin{align}
R(p_m,\gamma)
=p_m\int_0^\gamma S_{p_m}(u)\,\dd u
\le R_m^{(J)}(\gamma)
&\le p_m\Delta S_{p_m}(u_0)+p_m\int_0^{u_{n-1}}S_{p_m}(u)\,\dd u\nonumber\\
&\le p_m\Delta+R(p_m,\gamma),
\label{eq:riemann-telescope}
\end{align}
where the last inequality uses $S_{p_m}(u)\le1$ for any $u$ and $u_{n-1}<\gamma$. This proves \pref{eq:threshold-integration-error}.
\end{proof}
 
The final lemma controls the cumulative bonus, and hence the cumulative selected width $\sum_{t=1}^T w_t(m_t,\gamma_t)=2\sum_{t=1}^T B_{t,m_t}(\gamma_t)$. Each threshold estimate pools all observations at its price whose inventories exceed that threshold. Consequently, every radius appearing in the selected bonus is charged to an update of its shared count.
 
\begin{lemma}\label{lem:main-bonus-control}
On every sample path, the sequence generated by \pref{alg:threshold-ucb} satisfies
\begin{align}
\sum_{t=1}^T B_{t,m_t}(\gamma_t)
\le3PD\sqrt{(K+1)TL}.
\label{eq:main-bonus-control}
\end{align}
\end{lemma}
 
\begin{proof}
For each pair $(m,j)$, every selected bonus containing $C^{\mathrm{th}}_{t,m,j}$ also increments $N_{t,m,j}$ (Line~\ref{algline:threshold-count-N} of \pref{alg:threshold-ucb}). The counts before these visits are therefore $0,\ldots,N_{T+1,m,j}-1$. As in the proof of \pref{thm:double-grid-upper}, we use $\sum_{r=0}^{n-1}(\max\{1,r\})^{-1/2}\le2\sqrt n$ for every integer $n\ge1$, with an empty sum when $n=0$. Therefore, on every sample path, we know that
\begin{align*}
&\sum_{t=1}^T B_{t,m_t}(\gamma_t)\\
&=\Delta\sum_{j=1}^J\sum_{t:u_{j-1}<\gamma_t}p_{m_t}C^{\mathrm{th}}_{t,m_t,j}
\tag{definition of $B_{t,m_t}$}\\
&\le P\Delta\sqrt{\frac L2}
\sum_{j=1}^J\sum_{m=1}^{K+1}
\sum_{r=0}^{N_{T+1,m,j}-1}\frac{1}{\sqrt{\max\{1,r\}}}
\tag{grouping visits by $(m,j)$}\\
&\le3P\Delta\sqrt L\sum_{j=1}^J\sum_{m=1}^{K+1}\sqrt{N_{T+1,m,j}}
\tag{summing the radii over visits}\\
&\le3P\Delta\sqrt{(K+1)L}\sum_{j=1}^J\sqrt{\sum_{m=1}^{K+1}N_{T+1,m,j}}
\tag{Cauchy--Schwarz over prices}\\
&=3P\Delta\sqrt{(K+1)L}\sum_{j=1}^J\sqrt{\sum_{t=1}^T\one\{u_{j-1}<\gamma_t\}}
\tag{one price update per eligible threshold}\\
&\le3PD\sqrt{(K+1)TL},
\end{align*}
where the last inequality uses $J\Delta=D$ and the fact that each threshold is eligible in at most $T$ rounds.
Note that unlike the per-bin counts in \pref{eq:double-grid-radius-sum}, $N_{T+1,m,j}$ pools observations from all inventories exceeding $u_{j-1}$, so no separate count remains for an inventory bin. This proves \pref{eq:main-bonus-control}.
\end{proof}

\section{Lower Bound}\label{app:lower}

\pref{sec:algorithm} gives $\otil(T^{3/4})$ expected regret for Double-Grid-UCB (\pref{alg:double-grid-ucb}) and $\otil(T^{2/3})$
expected regret for Threshold-UCB (\pref{alg:threshold-ucb}). We now show that the latter exponent is optimal for
the nonparametric model. The hard instances use constant inventory
and binary demand, so censoring disappears. This restriction embeds classical stochastic
posted pricing~\citep{kleinberg2003value} in our model.

\begin{theorem}\label{thm:lower}
There is a universal constant $c>0$ such that, for all sufficiently large
$T$, every dynamic-pricing policy admits an instance satisfying
\pref{sec:model} and \pref{assum:monotone-demand} such that $\E[\Reg_T]\ge cT^{2/3}$.
\end{theorem}

\begin{proof}
It suffices to show that the classical unit-demand pricing problem studied by \citet{kleinberg2003value} is a special case of our model. Their lower-bound construction shows that, for every pricing policy and all sufficiently large $T$, there exists a fixed valuation distribution on $[0,1]$ under which expected regret relative to a price maximizing expected revenue is at least $cT^{2/3}$, for a universal constant $c>0$.

In this classical problem, buyer valuations $V_1,\ldots,V_T$ are drawn i.i.d. from an unknown distribution $\nu$ on $[0,1]$, independently of the learner's random seed. Denote by $\Prob_\nu$ and $\E_\nu$ probability and expectation, respectively, over the valuation sequence and the learner's randomization under valuation law $\nu$. At each round, the learner posts a price $p_t\in[0,1]$ based on past observations, observes only the purchase indicator $\one\{V_t\ge p_t\}$, and earns revenue $p_t\one\{V_t\ge p_t\}$. For any fixed price $p$, the purchase probability is $\Prob_\nu(V_t\ge p)=\nu([p,1])$, which does not depend on $t$. Hence, the expected regret is
\begin{equation*}
T\max_{p\in[0,1]}p\Prob_\nu(V_t\ge p)
-\E_\nu\left[\sum_{t=1}^T p_t\one\{V_t\ge p_t\}\right].
\end{equation*}

To embed this problem in our model, set $P=D=1$, choose $\gamma_t=1$ in every round, and define $Y_t(p)\triangleq\one\{V_t\ge p\}$. The potential-demand functions are stationary and independent of the pre-demand information $\calG_t$, as required by \pref{sec:model}. Moreover, for every $\gamma\in[0,1]$, we have $s(p,\gamma)=\gamma\Prob_\nu(V_t\ge p)$, which is nonincreasing in $p$ and therefore satisfies \pref{assum:monotone-demand}. Since demand is binary, $X_t=Y_t(p_t)$, so the observations and realized revenues coincide with those in the classical problem. Finally, constant inventory gives $R(p,1)=p\Prob_\nu(V_t\ge p)$, and conditioning on $\calG_t$ shows that $\E_\nu[\Reg_T]$ equals the expected regret displayed above. The classical lower bound therefore applies to this instance, proving the theorem.
\end{proof}
\section{Experimental Details and Additional Additive Results}
\label{app:experiment-details}

This section provides the experimental details referred to in
\pref{sec:experiments}. We first describe the benchmark algorithms, then specify the inventory processes and their
relation to C20CB's assumptions, give the shared tuning and evaluation
protocol, and finally present results for all 24 additive-noise environments.
\pref{app:multiplicative} gives the multiplicative-noise specifications and
the corresponding additional results.

\subsection{Benchmark Algorithms}\label{app:benchmark}

Here we provide the omitted details of the benchmark algorithms used in our experiments. We first describe KM-UCB and Naive-UCB, which use different estimators on the same price grid, and then present the doubling version of C20CB.

KM-UCB and Naive-UCB use the uniform price grid $\calP_K$ and a tunable exploration parameter $\alpha>0$. Both select the price with the largest estimated revenue plus an exploration bonus, breaking ties in favor of the higher price. Their hyperparameters are selected using the tuning protocol described in \pref{app:reproducibility-details}.

\paragraph{KM-UCB.}
KM-UCB uses the Kaplan--Meier estimator described in \pref{sec:threshold-construction} to estimate demand at each price. As discussed there, adaptive price selection means that the retained observations need not satisfy the i.i.d. sampling assumptions required by the classical confidence guarantee. We therefore implement KM-UCB as a heuristic benchmark: it integrates the estimated survival function up to the current inventory and adds an exploration bonus, as summarized in \pref{alg:benchmark-km-ucb}. We do not establish a regret guarantee for this benchmark under our adaptive sampling protocol.

\begin{algorithm2e}[H]
\caption{KM-UCB}
\label{alg:benchmark-km-ucb}
\KwIn{Grid size $K$ and exploration parameter $\alpha>0$}
Set $p_m=(m-1)P/K$ and initialize an empty observation list $\mathcal H_{1,m}$ for each $m\in[K+1]$\;
\For{$t=1,\ldots,T$}{
    Observe $\gamma_t$\;
    \ForEach{$m\in[K+1]$}{
        Set $N_{t,m}=|\mathcal H_{t,m}|$ and let $\mathcal Z_{t,m}$ contain the distinct sales in $\mathcal H_{t,m}$\;
        \ForEach{$z\in\mathcal Z_{t,m}$}{
            Set $n_{t,m}(z)=\sum_{(x,e)\in\mathcal H_{t,m}}\one\{x\ge z\}$ and $d_{t,m}(z)=\sum_{(x,e)\in\mathcal H_{t,m}}e\,\one\{x=z\}$\;
        }
        Set $\widehat S^{\mathrm{km}}_{t,m}(u)=\prod_{z\in\mathcal Z_{t,m}:\,z\le u}(1-d_{t,m}(z)/n_{t,m}(z))$, with an empty product equal to one\;
        Let $b_0,\ldots,b_q$ be the sorted distinct points in $\{0,\gamma_t\}\cup\{z\in\mathcal Z_{t,m}:0<z<\gamma_t\}$ and set $\widehat R^{\mathrm{km}}_{t,m}(\gamma_t)=p_m\sum_{r=0}^{q-1}(b_{r+1}-b_r)\widehat S^{\mathrm{km}}_{t,m}(b_r)$, with an empty sum equal to zero\;
        Set $\mathsf U_t^{\mathrm{km}}(m,\gamma_t)=\widehat R^{\mathrm{km}}_{t,m}(\gamma_t)+PD\sqrt{\alpha\log(\max\{2,T\})/(1+N_{t,m})}$\;
    }
    Choose the largest $m_t\in\arg\max_{m\in[K+1]}\mathsf U_t^{\mathrm{km}}(m,\gamma_t)$\;
    Post $p_t=p_{m_t}$ and observe $X_t$\;
    Set $\mathcal H_{t+1,m}=\mathcal H_{t,m}$ for all $m\in[K+1]$ and append $(X_t,\one\{X_t<\gamma_t\})$ to $\mathcal H_{t+1,m_t}$\;
}
\end{algorithm2e}

\paragraph{Naive-UCB.}
Naive-UCB uses the same price-selection rule but treats every sale as a fully observed demand realization. Specifically, it estimates revenue at the current inventory by averaging $p_m\min\{X_s,\gamma_t\}$ over previous sales at price $p_m$, as summarized in \pref{alg:benchmark-naive-ucb}. Thus, it ignores the information lost to earlier stockouts and can underestimate revenue at larger inventories.

\begin{algorithm2e}[H]
\caption{Naive-UCB}
\label{alg:benchmark-naive-ucb}
\KwIn{Grid size $K$ and exploration parameter $\alpha>0$}
Set $p_m=(m-1)P/K$ and initialize an empty sales list $\mathcal H_{1,m}$ for each $m\in[K+1]$\;
\For{$t=1,\ldots,T$}{
    Observe $\gamma_t$\;
    \ForEach{$m\in[K+1]$}{
        Set $N_{t,m}=|\mathcal H_{t,m}|$\;
        \lIf{$N_{t,m}=0$}{
            Set $\widehat R^{\mathrm{naive}}_{t,m}(\gamma_t)=p_m\gamma_t$
        }
        \lElse{
            Set $\widehat R^{\mathrm{naive}}_{t,m}(\gamma_t)=\frac{p_m}{N_{t,m}}\sum_{x\in\mathcal H_{t,m}}\min\{x,\gamma_t\}$
        }
        Set $\mathsf U_t^{\mathrm{naive}}(m,\gamma_t)=\widehat R^{\mathrm{naive}}_{t,m}(\gamma_t)+PD\sqrt{\alpha\log(\max\{2,T\})/(1+N_{t,m})}$\;
    }
    Choose the largest $m_t\in\arg\max_{m\in[K+1]}\mathsf U_t^{\mathrm{naive}}(m,\gamma_t)$\;
    Post $p_t=p_{m_t}$ and observe $X_t$\;
    Set $\mathcal H_{t+1,m}=\mathcal H_{t,m}$ for all $m\in[K+1]$ and append $X_t$ to $\mathcal H_{t+1,m_t}$\;
}
\end{algorithm2e}

\paragraph{C20CB-doubling.}
The fixed-horizon benchmark runs C20CB \citep{xu2026dynamic} with horizon $T$. Its doubling version instead starts a fresh C20CB instance on each epoch of length $1,2,4,\ldots$, so its online rule does not require the terminal horizon. As shown in \pref{alg:benchmark-c20cb-doubling}, each instance uses its epoch length to set the exploration length and noise-grid resolution. Only the learner is restarted and the inventory and demand process continues unchanged. If the experiment ends during an epoch, that epoch is truncated.

\begin{algorithm2e}[H]
\caption{C20CB-doubling}
\label{alg:benchmark-c20cb-doubling}
\KwIn{Exploration multiplier $s_\tau>0$, grid multiplier $s_M>0$, and confidence multiplier $c_{\rm conf}>0$}
Initialize epoch length $H=1$ and completed rounds within the epoch $h=0$\;
\For{$t=1,2,\ldots$}{
    \lIf{$h=H$}{
        Set $H\leftarrow 2H$ and $h=0$
    }
    \If{$h=0$}{
        Set $\tau=\min\{H,\max\{1,\lceil s_\tau\sqrt H\rceil\}\}$ and $M=\max\{1,\lceil s_MH^{1/4}\rceil\}$\;
        Initialize a fresh instance $\mathcal A=\textsc{C20CB}(H,\tau,M,c_{\rm conf})$\;
    }
    Observe $\gamma_t$ and obtain $p_t=\mathcal A.\textsc{SelectPrice}(\gamma_t)$, post $p_t$, and observe $X_t$\;
    Call $\mathcal A.\textsc{Update}(\gamma_t,p_t,X_t)$\;
    Set $h\leftarrow h+1$\;
}
\end{algorithm2e}

\subsection{Inventory configurations and assumption checks}
\label{app:additive-assumption-map}
In this section, we provide the complete inventory configurations used in
\pref{sec:experiments} and verify which demand-inventory combinations satisfy
the conditions of C20CB. The demand curves and additive-noise distributions are specified in
\pref{sec:experiments}. \pref{tab:inventory-configs} completes the inventory
specifications, using the same support and distribution codes as the figures.
For the piecewise processes, we use $\lceil T^{1/3}\rceil$ blocks of length
$\lceil T/\lceil T^{1/3}\rceil\rceil$, truncating the last block at $T$.
The regimes cycle from low to central to high, and inventories are sampled
independently within each block. All inventory paths are generated before the interaction starts.

\begin{table}[H]
\centering
\caption{Inventory processes used in both noise suites. The codes combine
support pattern with distribution: \texttt{U} denotes uniform and \texttt{G}
denotes clipped Gaussian. Piecewise processes cycle through three regimes.}
\label{tab:inventory-configs}
\begin{tabular}{clp{0.62\textwidth}}
\toprule
Code & Type & Exact configuration\\
\midrule
\texttt{1-U} & Uniform & $\gamma_t\sim\mathrm{Unif}[0.86,1.29]$.\\
\texttt{2-U} & Uniform & $\gamma_t\sim\mathrm{Unif}[0.02,1.75]$.\\
\texttt{1-G} & Clipped Gaussian & $\gamma_t=\operatorname{clip}\{N(1.075,0.14^2),0.86,1.29\}$.\\
\texttt{2-G} & Clipped Gaussian & $\gamma_t=\operatorname{clip}\{N(0.85,0.55^2),0.02,1.75\}$.\\
\texttt{3-U} & Piecewise uniform & Cycle $[0.02,0.30]$, $[0.86,1.29]$, and $[1.45,1.75]$.\\
\texttt{3-G} & Piecewise Gaussian & Cycle clipped $N(0.12,0.05^2)$,
$N(1.075,0.14^2)$, and $N(1.60,0.06^2)$ on the same three supports.\\
\bottomrule
\end{tabular}
\end{table}
We next examine more closely which combinations satisfy or violate the conditions of C20CB in~\citet{xu2026dynamic}.
Both additive linear models use noise $\varepsilon_t\sim\mathrm{Unif}[-0.10,0.10]$, so $c=0.10$ in the context of~\citet{xu2026dynamic}  is the tight noise bound.
For Demand A, the true parameters are $(a,b)=(1.4,0.8)$, and the experiments supply $(a_{\max},b_{\min},b_{\max},c)=(1.5,0.75,0.85,0.10)$.
These bounds require $\gamma_t\in(0.85,1.30)$ and $\gamma_t>0.20$, which both the central uniform and central clipped-Gaussian inventory processes satisfy. Even the tightest valid bounds only widen the interval to $(0.70,1.30)$ and the broad and piecewise inventory supports extend outside this interval, so their violations are not caused by loose parameter bounds.

For Demand B, changing the price over $[0,1]$ changes mean demand by only $0.01$, less than the noise support width of $0.20$. Thus, no inventory level can lie below all demand realizations at the lowest price and above all demand realizations at the highest price, as required by C20CB. Specifically, any valid bounds satisfy $a_{\max}\ge a$, $b_{\min}\le b$, and $c\ge0.10$. With $(a,b)=(1.10,0.01)$ and $P=1$, C20CB requires both $\gamma_t>a_{\max}-b_{\min}P+c\ge1.19$ and $\gamma_t<a-c\le1$. These requirements are incompatible for every valid choice of bounds. C20CB is not tested on Demands C and D because their conditional means are nonlinear.

\subsection{Tuning and evaluation protocol}\label{app:reproducibility-details}
 
All experiments use prices in $[0,P]$ with $P=1$ and horizon $T=10{,}000$. We use the same protocol for every environment, that is, for every combination of demand family, inventory process, and noise model in \pref{sec:experiments}.
 
\paragraph{Hyperparameter tuning.}
In each environment, every method is tuned separately with the same budget of $20$ hyperparameter configurations. The first eight configurations include the reference configuration in \pref{tab:tuning-spaces}; the remaining twelve are proposed sequentially by Gaussian-process Bayesian optimization based on the results so far. Each configuration is run on the same five tuning paths, and we select the configuration with the smallest mean cumulative regret at round $T$ over these paths. The selected configuration is then evaluated on $30$ held-out paths that are not used for tuning. Thus, every method receives the same number of tuning runs and evaluation runs.
 
\pref{tab:tuning-spaces} lists the tuned hyperparameters for the additive-noise experiments. For Threshold-UCB, $K$ and $J$ are the numbers of price-grid intervals and thresholds in \pref{alg:threshold-ucb}, and $c_{\rm conf}$ scales the confidence radii $C^{\mathrm{th}}_{t,m,j}$. For KM-UCB and Naive-UCB (\pref{alg:benchmark-km-ucb} and \pref{alg:benchmark-naive-ucb}), $K$ is the price-grid size and $\alpha$ scales the exploration bonus. For C20CB and C20CB-doubling (\pref{alg:benchmark-c20cb-doubling}), $s_\tau$ scales the length $\tau$ of the initial exploration phase, $s_M$ scales the size $M$ of the noise grid, and $c_{\rm conf}$ scales the confidence widths. The reference value of $K$ and $J$ is $\lceil T^{1/3}\rceil=22$, and the reference values of the other hyperparameters are given in parentheses in \pref{tab:tuning-spaces}. Each search range spans from half to twice the reference value and is searched on a logarithmic scale; $K$ and $J$ are rounded to the nearest integer, with ties rounded up.
 
\begin{table}[H]
\centering
\caption{Hyperparameters tuned in the additive-noise experiments with $T=10{,}000$.
Reference values, which are included among the initial configurations, are given in parentheses.}
\label{tab:tuning-spaces}
\small
\begin{tabular}{p{0.19\textwidth}p{0.20\textwidth}p{0.51\textwidth}}
\toprule
Method & Hyperparameters & Search ranges (reference values)\\
\midrule
Threshold-UCB & $K,J,c_{\rm conf}$ &
$K,J\in\{11,\ldots,44\}$ ($22$ each);
$c_{\rm conf}\in[0.5,2]$ ($1$).\\
C20CB & $s_\tau,s_M,c_{\rm conf}$ &
$s_\tau,s_M,c_{\rm conf}\in[0.5,2]$ ($1$ each), with
$\tau=\lceil s_\tau\sqrt T\rceil$ and
$M=\lceil s_MT^{1/4}\rceil$.\\
C20CB-doubling & $s_\tau,s_M,c_{\rm conf}$ & Same as C20CB, with $T$ replaced
by the epoch length $H\in\{1,2,4,\ldots\}$ and $\tau$ capped at $H$.\\
KM-UCB & $K,\alpha$ &
$K\in\{11,\ldots,44\}$ ($22$);
$\alpha\in[1,4]$ ($2$).\\
Naive-UCB & $K,\alpha$ & Same as KM-UCB.\\
\bottomrule
\end{tabular}
\end{table}
 
\paragraph{Fixed implementation choices.}
The remaining settings are fixed across environments. Some of them differ from the choices used in our analysis; we list them explicitly.
\begin{itemize}
\item \emph{Threshold-UCB.} The threshold grid uses $D=1.8$, which upper-bounds every inventory level in \pref{tab:inventory-configs}, so $\Delta=1.8/J$. The index omits the term $j=1$, that is, the threshold $u_0=0$, and sums only over the positive thresholds $u_1,\ldots,u_{J-1}$ below $\gamma_t$. With $\delta=0.05$, the confidence radius is
\[
C^{\mathrm{th}}_{t,m,j}=c_{\rm conf}\sqrt{\frac{\log(2(K+1)JT/\delta)}{2N_{t,m,j}}},\]
if $N_{t,m,j}>0$, and $C^{\mathrm{th}}_{t,m,j}=c_{\rm conf}$ if $N_{t,m,j}=0$ in place of the radius in \pref{alg:threshold-ucb}. Ties in the index are broken in favor of the higher price, as in KM-UCB and Naive-UCB.
\item \emph{C20CB and C20CB-doubling.}
The original C20CB analysis involves several problem-dependent constants controlling the uncertainty from the initial parameter-estimation phase. To keep the benchmark implementation tractable and avoid tuning a large number of C20CB-specific constants, we tie these first-stage uncertainty corrections to a single fixed coefficient, set to $0.25$, for the corresponding $1/\sqrt{\tau}$ terms. The resulting confidence widths are further scaled by the tuned multiplier $c_{\rm conf}$. Both variants also clip the parameter estimates to the supplied bounds $(a_{\max},b_{\min},b_{\max},c)$ in \pref{app:additive-assumption-map}, and use the fallback price $0.5$ if the checks after the initial exploration phase fail.
\end{itemize}
The modifications for multiplicative noise are described in \pref{app:multiplicative}.
 
\paragraph{Evaluation.}
Within each tuning or held-out path, all methods face the same inventory sequence $\gamma_1,\ldots,\gamma_T$ and the same underlying noise draws, so differences in performance come only from the posted prices. The resulting demands can still differ because they depend on the price. Tuning and held-out paths use separate random seeds. We report the cumulative expected regret
\[
\sum_{t=1}^T\left[\max_{p\in\calP_{2000}}R(p,\gamma_t)-R(p_t,\gamma_t)\right],
\]
where $\calP_{2000}$ is the uniform grid of $2001$ prices in $[0,1]$. This fine grid approximates the continuous oracle in \pref{eq:continuous-regret} and is independent of the learner's price grid. The expected revenue $R(p,\gamma)$ is computed in closed form whenever available. For additive Demand D, the revenue is instead approximated numerically using a fixed $41$-point quadrature rule. Each curve shows the mean over the $30$ held-out paths, and the shaded band shows the mean plus or minus one standard deviation, with its lower edge clipped at zero.

\section{Additional additive results}
\label{app:complete-additive-results}

In this section, we provide the complete results for the additive-noise experiments. \pref{fig:additive-A}--\pref{fig:additive-D} show one demand family each, with the six inventory processes arranged in the same $2\times3$ layout: columns vary the inventory support and rows vary the inventory distribution.

These figures complete the additive comparisons in
\pref{fig:diagnostic-tuples}. Threshold-UCB has the smallest mean final regret in $22$ of the $24$ environments. Fixed-horizon C20CB has the smallest mean in the two central-inventory environments for Demand A, where its linear-additive and inventory-support conditions hold. For all other environments, the figures show the effect of relaxing those conditions.

\begin{figure}[H]
\centering
\includegraphics[width=\textwidth]{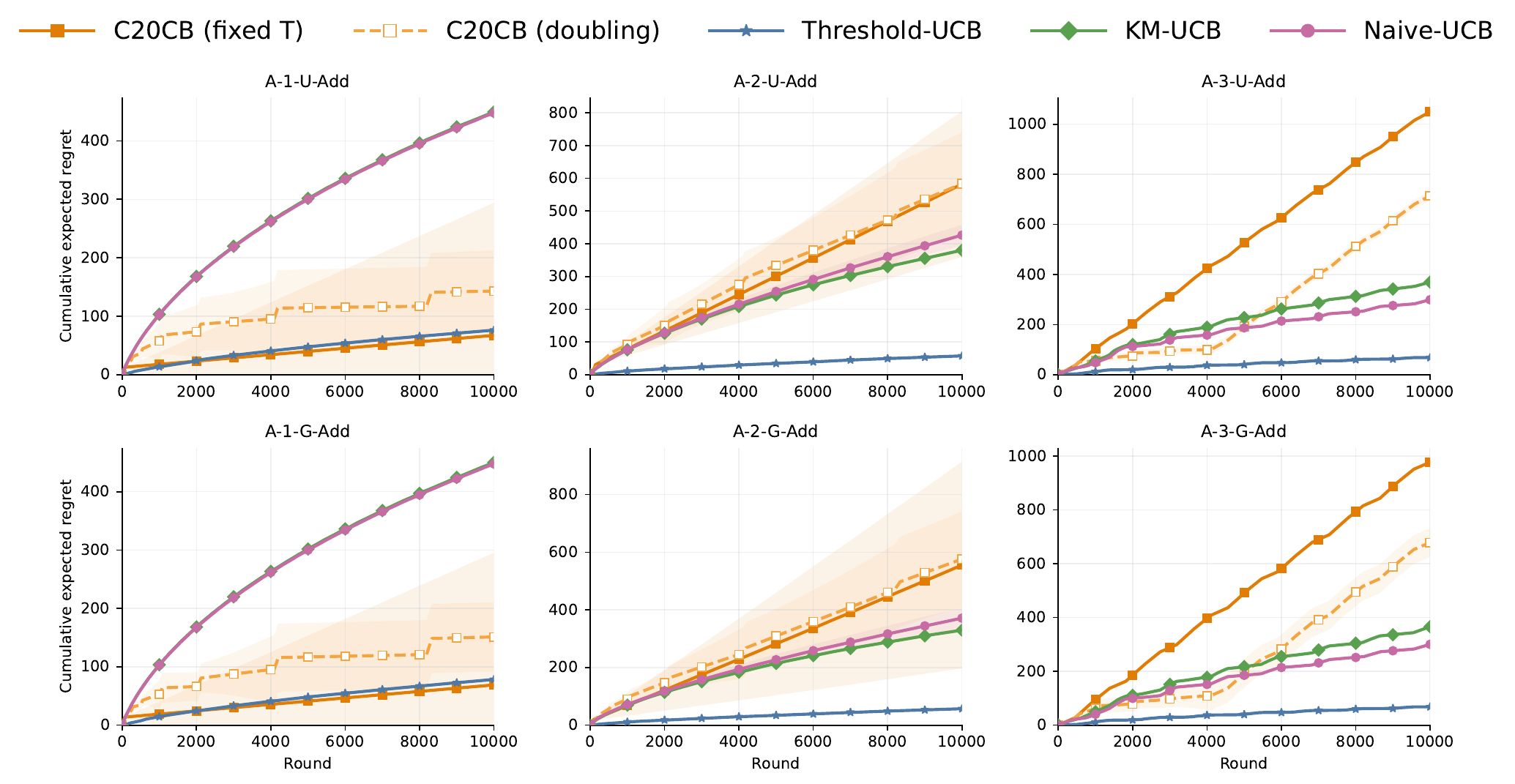}
\caption{Additive-noise Demand A. Columns correspond to the central (\texttt{1}), broad (\texttt{2}), and piecewise-changing (\texttt{3}) inventory supports; the top and bottom rows use uniform (\texttt{U}) and clipped-Gaussian (\texttt{G}) inventory distributions, respectively. Only \texttt{A-1-U-Add} and \texttt{A-1-G-Add} satisfy C20CB's inventory conditions; the broad and piecewise processes test performance when those conditions fail.}
\label{fig:additive-A}
\end{figure}

\begin{figure}[H]
\centering
\includegraphics[width=\textwidth]{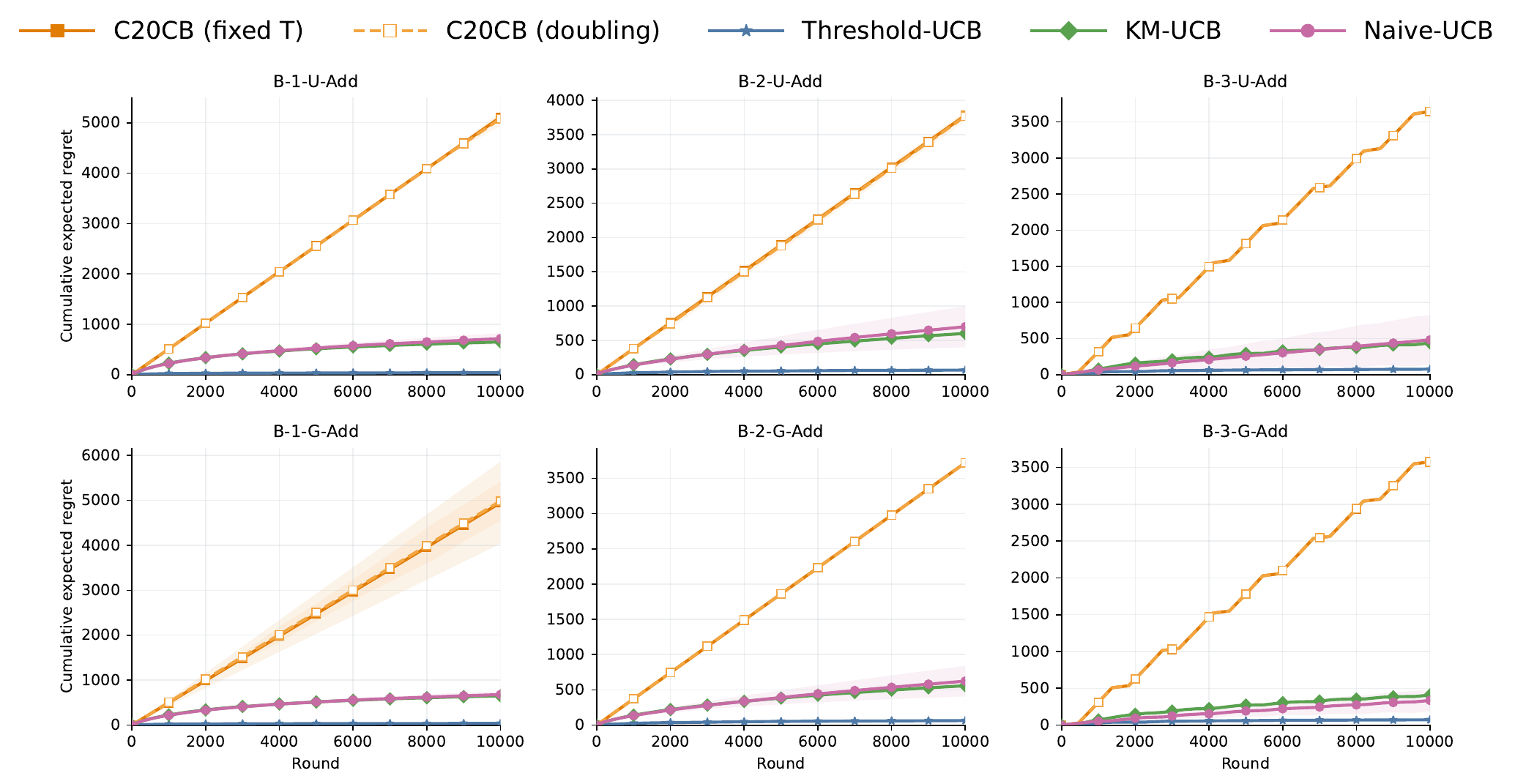}
\caption{Additive-noise Demand B. Columns correspond to the central (\texttt{1}), broad (\texttt{2}), and piecewise-changing (\texttt{3}) inventory supports; the top and bottom rows use uniform (\texttt{U}) and clipped-Gaussian (\texttt{G}) inventory distributions, respectively.}
\label{fig:additive-B}
\end{figure}

\begin{figure}[H]
\centering
\includegraphics[width=\textwidth]{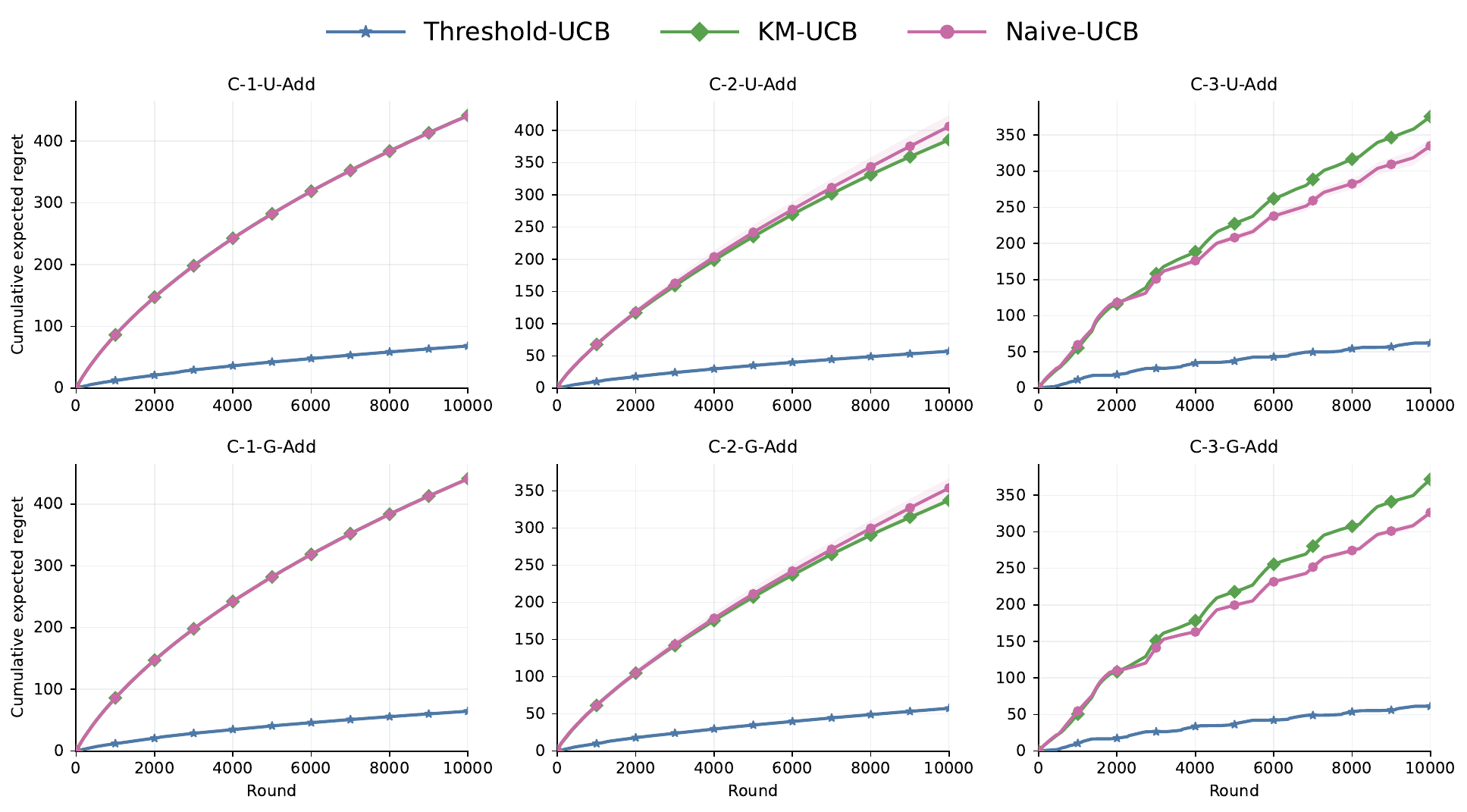}
\caption{Additive-noise Demand C. Columns correspond to the central (\texttt{1}), broad (\texttt{2}), and piecewise-changing (\texttt{3}) inventory supports; the top and bottom rows use uniform (\texttt{U}) and clipped-Gaussian (\texttt{G}) inventory distributions, respectively. C20CB is omitted because the mean demand is nonlinear.}
\label{fig:additive-C}
\end{figure}

\begin{figure}[H]
\centering
\includegraphics[width=\textwidth]{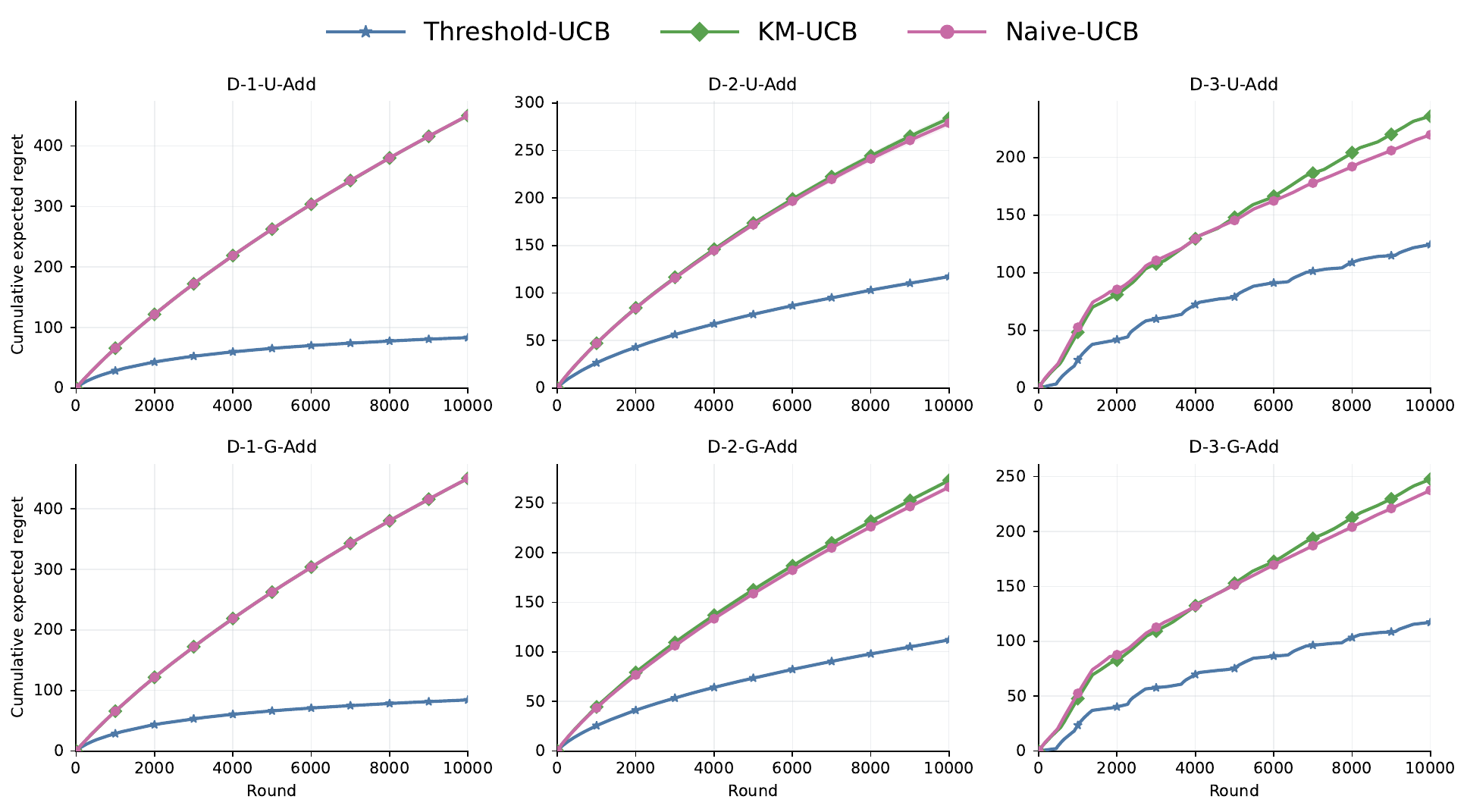}
\caption{Additive-noise Demand D. Columns correspond to the central (\texttt{1}), broad (\texttt{2}), and piecewise-changing (\texttt{3}) inventory supports; the top and bottom rows use uniform (\texttt{U}) and clipped-Gaussian (\texttt{G}) inventory distributions, respectively. C20CB is omitted because the mean demand is nonlinear.}
\label{fig:additive-D}
\end{figure}
\section{Additional Results with Multiplicative Noise}
\label{app:multiplicative}

In this section, we provide the details and complete results for the multiplicative-noise experiments (\texttt{Mult}) in \pref{sec:experiments}. We use the same four demand curves $\mu(p)$ and the same six inventory processes in \pref{tab:inventory-configs}, so there are again $24$ environments. The only change is the noise model: unlike the additive case, the shape of the demand distribution now changes with price, even after centering at its mean. We first specify the noise model and check which assumptions it satisfies, then describe how the protocol in \pref{app:reproducibility-details} is adapted, and finally report the results.

\subsection{Noise specifications and assumption checks}
\label{app:mult-noise}

For each demand family, potential demand is $Y_t(p)=\mu(p)Z_t(p)$, where the multiplier $Z_t(p)$ takes two values:
\[
Z_t(p)=
\begin{cases}
\ell, & \text{with probability } q(p),\\[2pt]
h(p)\triangleq\dfrac{1-\ell q(p)}{1-q(p)}, & \text{with probability } 1-q(p).
\end{cases}
\]
Here, $\ell<1$ is a fixed low multiplier and $q(p)<1$ is the price-dependent probability of the low state. The high multiplier $h(p)\ge1$ is chosen so that $\E[Z_t(p)]=1$. Hence, expected potential demand equals $\mu(p)$ at every price, exactly as in the additive case, and only the shape of the distribution changes. \pref{tab:multiplicative-configs} specifies $\ell$ and $q(p)$ for each family, where, for Demand B, $x(p)\triangleq\operatorname{clip}((p-0.66)/0.18,0,1)$, so that $q(p)$ increases smoothly from $0$ to $0.95$ as $p$ increases from $0.66$ to $0.84$. For Demands A and C, $q(p)=0$ for $p\le0.65$, so demand is deterministic at low prices and becomes random only at high prices.

\begin{table}[H]
\centering
\caption{Multiplier distributions for multiplicative noise. The last column gives the demand bound $D_{\rm exp}\ge\max_{p\in[0,1]}\mu(p)h(p)$, which upper-bounds every realization of potential demand.}
\label{tab:multiplicative-configs}
\begin{tabular}{lccc}
\toprule
Demand & Low multiplier $\ell$ & Low-state probability $q(p)$ & Demand bound $D_{\rm exp}$\\
\midrule
A & $0.20$ & $2(p-0.65)_+$ & $1.8$\\
B & $0.63$ & $0.95\bigl(3x(p)^2-2x(p)^3\bigr)$ & $9$\\
C & $0.20$ & $2(p-0.65)_+$ & $1.35$\\
D & $0.20$ & $\min\{0.85,1.4p\}$ & $2.25$\\
\bottomrule
\end{tabular}
\end{table}

\paragraph{Monotone expected sales.}
All four environments have expected sales that are nonincreasing in price, satisfying \pref{assum:monotone-demand}. Fix $p\le p'$ and $\gamma\ge0$. Since $\mu(p)$ is nonincreasing and $q(p)$ is nondecreasing, we have $\mu(p)\ge\mu(p')$ and $q(p)\le q(p')$. Recall that $Z_t(p)$ equals $\ell$ with probability $q(p)$ and $h(p)$ otherwise, where $h(p)$ is chosen so that $\E[Z_t(p)]=1$. Hence, we have $$(1-q(p))h(p)=1-q(p)\ell=(q(p')-q(p))\ell+(1-q(p'))h(p').$$ Since $z\mapsto\min\{\mu(p')z,\gamma\}$ is concave, this identity implies $$(1-q(p))\min\{\mu(p')h(p),\gamma\}\ge(q(p')-q(p))\min\{\mu(p')\ell,\gamma\}+(1-q(p'))\min\{\mu(p')h(p'),\gamma\}.$$ Adding $q(p)\min\{\mu(p')\ell,\gamma\}$ to both sides gives $\E[\min\{\mu(p')Z_t(p),\gamma\}]\ge\E[\min\{\mu(p')Z_t(p'),\gamma\}]$. Finally, since $\mu(p)\ge\mu(p')$ and $Z_t(p)\ge0$, we have $$\E[\min\{\mu(p)Z_t(p),\gamma\}]\ge\E[\min\{\mu(p')Z_t(p),\gamma\}].$$
Therefore, $s(p,\gamma)\ge s(p',\gamma)$.

\subsection{Changes to the tuning and evaluation protocol}
\label{app:mult-protocol}

We use the tuning and evaluation protocol in \pref{app:reproducibility-details}, including the budget of $20$ configurations on five tuning paths and the $30$ held-out paths, with the following changes.
\begin{itemize}
\item \emph{Demand bound.} Potential demand can now exceed the largest inventory level, and its range differs across families. We therefore replace $D=1.8$ by the family-specific bound $D_{\rm exp}$ in \pref{tab:multiplicative-configs}, both in Threshold-UCB's threshold grid and in the exploration bonuses $PD\sqrt{\alpha\log(\max\{2,T\})/(1+N_{t,m})}$ of KM-UCB and Naive-UCB. Since $Y_t(p)\le D_{\rm exp}$, we have $S_p(u)=0$ for every $u\ge D_{\rm exp}$, so thresholds in $[0,D_{\rm exp}]$ suffice even when inventory exceeds $D_{\rm exp}$.
\item \emph{Threshold count of Threshold-UCB.} To keep the threshold spacing $\Delta=D_{\rm exp}/J$ comparable across families, the reference value of $J$ is $J_0=\lceil D_{\rm exp}T^{1/3}\rceil$, and $J$ is tuned over the integers in $[\max\{2,\lfloor0.5J_0\rfloor\},\lceil2J_0\rceil]$. The ranges of $K$ and $c_{\rm conf}$ are unchanged.
\item \emph{C20CB and C20CB-doubling.} Since no true additive-noise bound exists, both variants additionally tune the noise bound $c\in[0.025,0.4]$ that they use internally, on a logarithmic scale with reference value $0.1$, and the fallback price in $[0.1,0.9]$, on a linear scale with reference value $0.5$. This tuning gives C20CB more flexibility but does not restore its common additive-noise assumption.
\end{itemize}
All other hyperparameter ranges and fixed settings are as in \pref{tab:tuning-spaces} and \pref{app:reproducibility-details}. For evaluation, the expected revenue has the closed form
\[
R(p,\gamma)=p\Bigl[q(p)\min\{\ell\mu(p),\gamma\}+\bigl(1-q(p)\bigr)\min\{h(p)\mu(p),\gamma\}\Bigr],
\]
and regret is computed against the same $2001$-point oracle grid $\calP_{2000}$ as in the additive case.

\subsection{Additional multiplicative results}
\label{sec:complete-multiplicative-results}

In this section, we report the complete results for the multiplicative-noise experiments. \pref{fig:mult-A}--\pref{fig:mult-D} each show one demand family, with the six inventory processes arranged in the same $2\times3$ layout as in \pref{app:complete-additive-results}. Threshold-UCB still has the smallest mean final regret in all $24$ environments. For Demand A, the mean is the same linear curve as in the additive reference and only the noise distribution changes, so the gap between Threshold-UCB and C20CB cannot be attributed to a nonlinear mean. For the nonlinear Demands C and D, we compare only Threshold-UCB, KM-UCB, and Naive-UCB.

\begin{figure}[H]
\centering
\includegraphics[width=\textwidth]{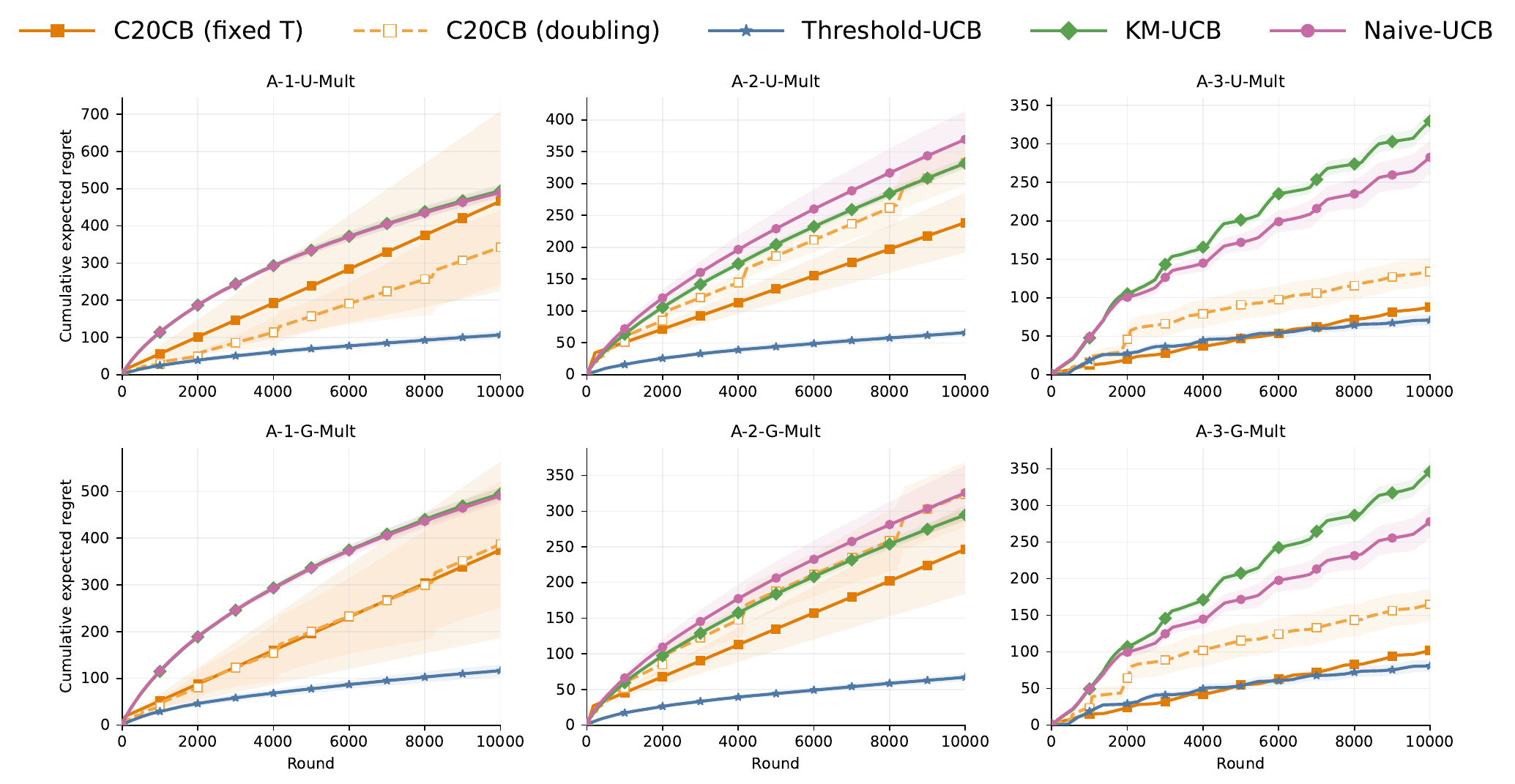}
\caption{Multiplicative-noise Demand A. Columns correspond to the central (\texttt{1}), broad (\texttt{2}), and piecewise-changing (\texttt{3}) inventory supports; the top and bottom rows use uniform (\texttt{U}) and clipped-Gaussian (\texttt{G}) inventory distributions, respectively. The mean matches additive Demand A, but the price-dependent noise violates C20CB's common additive-noise assumption.}
\label{fig:mult-A}
\end{figure}

\begin{figure}[H]
\centering
\includegraphics[width=\textwidth]{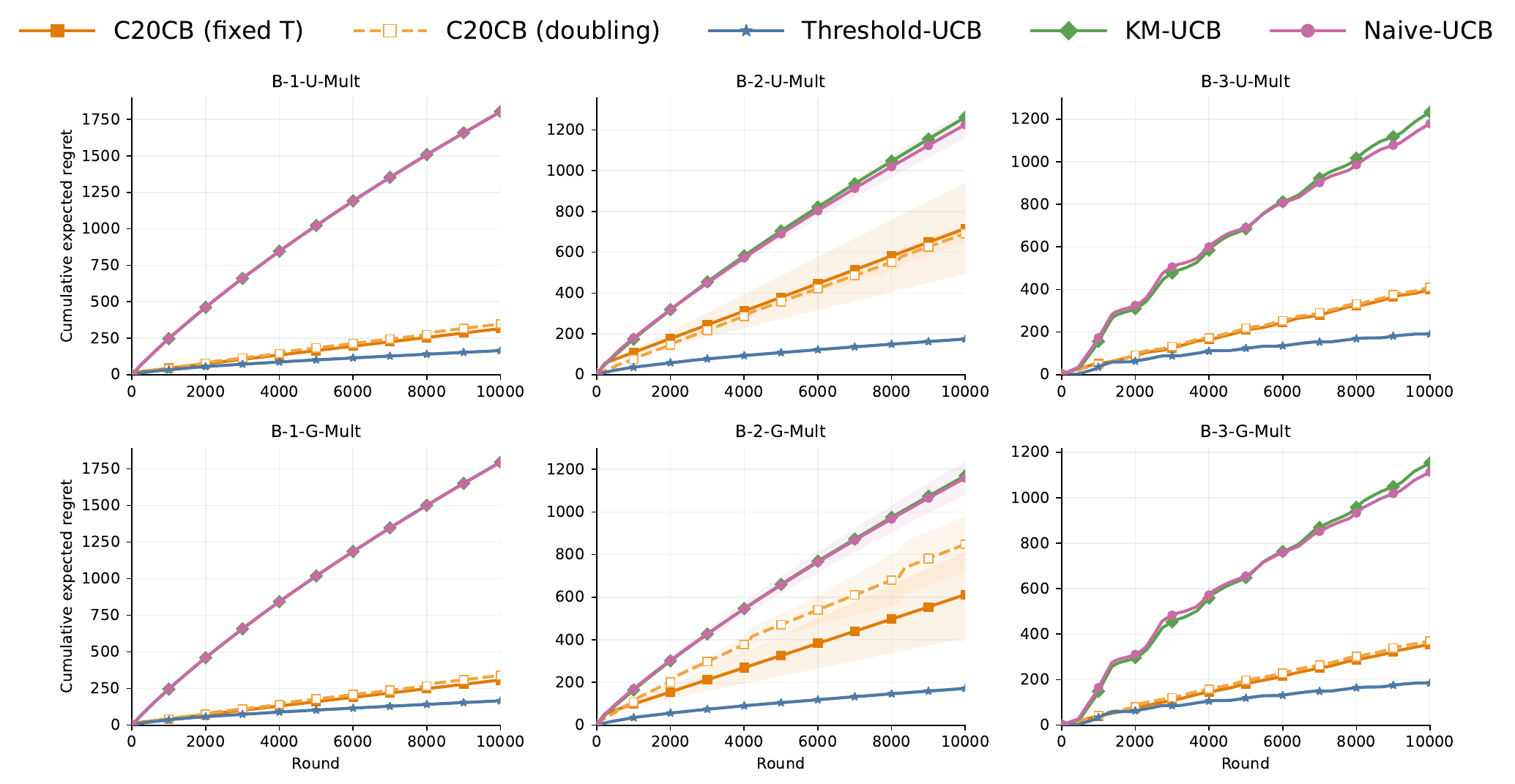}
\caption{Multiplicative-noise Demand B. Columns correspond to the central (\texttt{1}), broad (\texttt{2}), and piecewise-changing (\texttt{3}) inventory supports; the top and bottom rows use uniform (\texttt{U}) and clipped-Gaussian (\texttt{G}) inventory distributions, respectively. The nearly flat linear mean is combined with a price-dependent probability of low demand, which violates C20CB's common additive-noise assumption.}
\label{fig:mult-B}
\end{figure}

\begin{figure}[H]
\centering
\includegraphics[width=\textwidth]{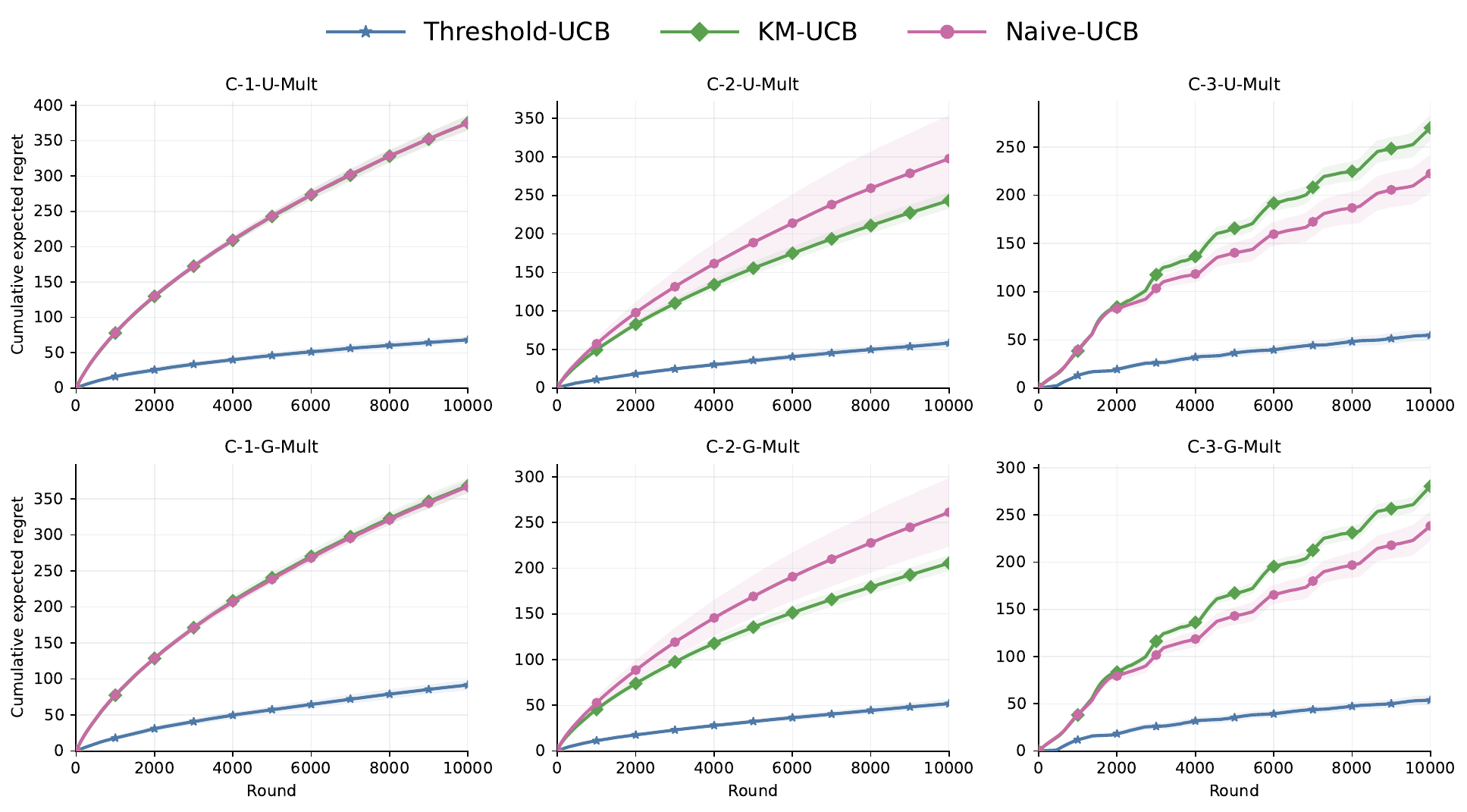}
\caption{Multiplicative-noise Demand C. Columns correspond to the central (\texttt{1}), broad (\texttt{2}), and piecewise-changing (\texttt{3}) inventory supports; the top and bottom rows use uniform (\texttt{U}) and clipped-Gaussian (\texttt{G}) inventory distributions, respectively. C20CB is omitted because the mean demand is nonlinear.}
\label{fig:mult-C}
\end{figure}

\begin{figure}[H]
\centering
\includegraphics[width=\textwidth]{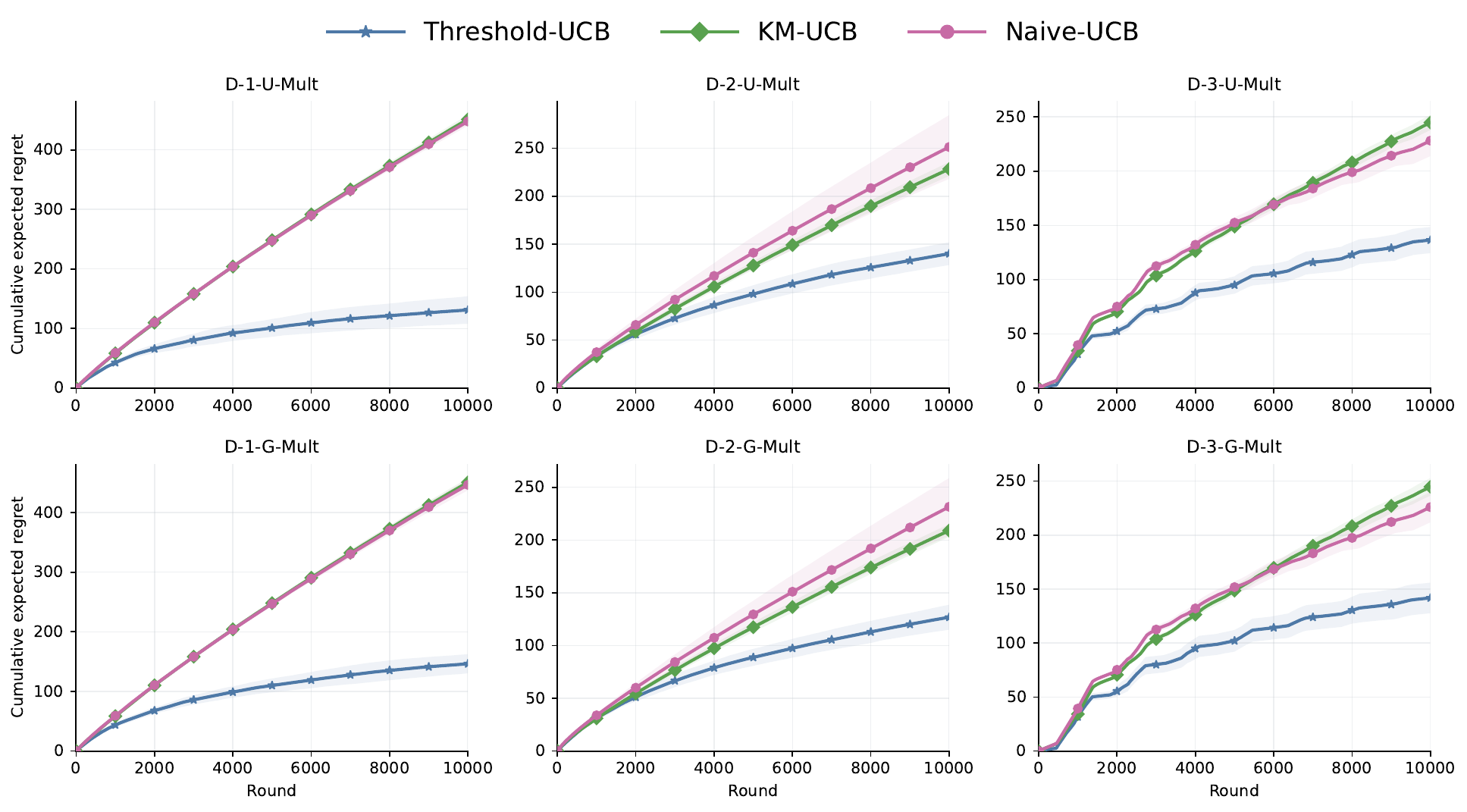}
\caption{Multiplicative-noise Demand D. Columns correspond to the central (\texttt{1}), broad (\texttt{2}), and piecewise-changing (\texttt{3}) inventory supports; the top and bottom rows use uniform (\texttt{U}) and clipped-Gaussian (\texttt{G}) inventory distributions, respectively. C20CB is omitted because the mean demand is nonlinear.}
\label{fig:mult-D}
\end{figure}
\section{Auxiliary Lemmas}\label{app:auxiliary}

\begin{lemma}[Hoeffding's inequality; \citealp{hoeffding1963probability}]\label{lem:hoeffding}
Let $W_1,\ldots,W_n$ be independent random variables taking values in $[a,b]$, where $a<b$, with common mean $\mu$. Then, for every $\epsilon>0$,
\begin{equation*}
\Prob\left(\left|\frac1n\sum_{r=1}^n W_r-\mu\right|\ge\epsilon\right)
\le2\exp\left(-\frac{2n\epsilon^2}{(b-a)^2}\right).
\end{equation*}
\end{lemma}

\end{document}